\documentclass[runningheads]{llncs}

\usepackage{microtype}
\usepackage{graphicx}
\usepackage{booktabs} 
\usepackage{natbib}

\usepackage[colorlinks=true,linkcolor=blue!80!black,
citecolor=green!50!black,urlcolor=magenta]{hyperref}

\usepackage{etoolbox}
\makeatletter
\pretocmd{\NAT@citex}{%
	\let\NAT@hyper@\NAT@hyper@citex
	\def\NAT@postnote{#2}%
	\setcounter{NAT@total@cites}{0}%
	\setcounter{NAT@count@cites}{0}%
	\forcsvlist{\stepcounter{NAT@total@cites}\@gobble}{#3}}{}{}
\newcounter{NAT@total@cites}
\newcounter{NAT@count@cites}
\def\NAT@postnote{}
\def\NAT@hyper@citex#1{%
	\stepcounter{NAT@count@cites}%
	\hyper@natlinkstart{\@citeb\@extra@b@citeb}#1%
	\ifnumequal{\value{NAT@count@cites}}{\value{NAT@total@cites}}
	{\ifNAT@swa\else\if*\NAT@postnote*\else%
		\NAT@cmt\NAT@postnote\global\def\NAT@postnote{}\fi\fi}{}%
	\ifNAT@swa\else\if\relax\NAT@date\relax
	\else\NAT@@close\global\let\NAT@nm\@empty\fi\fi
	\hyper@natlinkend}
\renewcommand\hyper@natlinkbreak[2]{#1}
\patchcmd{\NAT@citex}
{\ifNAT@swa\else\if*#2*\else\NAT@cmt#2\fi
	\if\relax\NAT@date\relax\else\NAT@@close\fi\fi}{}{}{}
\patchcmd{\NAT@citex}
{\if\relax\NAT@date\relax\NAT@def@citea\else\NAT@def@citea@close\fi}
{\if\relax\NAT@date\relax\NAT@def@citea\else\NAT@def@citea@space\fi}{}{}
\makeatother

\usepackage{amsmath}
\usepackage{amssymb}
\usepackage{mathtools}

\usepackage[capitalize,noabbrev]{cleveref}

\usepackage[utf8]{inputenc}
\usepackage[T1]{fontenc}    
\usepackage{dsfont} 
\usepackage{xcolor}
\usepackage{enumitem} 
\usepackage{subcaption}
\usepackage{wrapfig}
\usepackage{bm} 
\usepackage{multirow}
\usepackage{fancybox}

\usepackage{pgfplots}
\usepackage{pgfplotstable}
\usepgfplotslibrary{groupplots} 
\usepgfplotslibrary{fillbetween}
\pgfplotsset{compat=1.16}
\pgfplotsset{grid = major, grid style={gray!30!white}}

\definecolor{darkgreen}{rgb}{0.31, 0.47, 0.26}

\tikzset{directSaaStyle/.style={blue, mark=*,mark options={fill=blue}, line width=0.5mm}}
\tikzset{robustSaa10Style/.style={red, mark=square*,mark options={fill=red}, line width=0.5mm}}
\tikzset{alphaStyle/.style={black, dashed, mark=none, line width=0.5mm}}
\tikzset{confFillStyle/.style={opacity=0.2}}
\tikzset{confBorderAStyle/.style={name path=A, draw=none, mark=none}}
\tikzset{confBorderBStyle/.style={name path=B, draw=none, mark=none}}
\tikzset{robustSaa5Style/.style={darkgreen, mark=triangle*,mark options={fill=darkgreen}, line width=0.5mm}}
\tikzset{isolationForestStyle/.style={black, dashed,  mark=diamond*,mark options={fill=black, solid}, line width=0.5mm}}
\tikzset{naiveStyle/.style={dash dot, brown, mark=pentagon*, mark options={fill=brown}, mark size={3}, line width=0.5mm}}
\tikzset{lofStyle/.style={cyan, dashed, mark=diamond*, mark options={fill=cyan, solid}, line width=0.5mm}}
\tikzset{robxStyle/.style={violet, dashed, mark=o, mark options={fill=violet, solid}, line width=0.5mm}}

\DeclareMathOperator{\Prob}{\mathbb{P}}
\DeclareMathOperator{\CE}{\mathcal{C}} 
\DeclareMathOperator{\Ber}{Bernoulli} 
\DeclareMathOperator{\Binom}{Bin} 
\newcommand{\pStar}{p_{N, \alpha}^*} 

\begin{document}
\title{Don't Explain Noise: Robust Counterfactuals for Randomized Ensembles}
\titlerunning{Don't Explain Noise: Robust Counterfactuals for Randomized Ensembles}

\author{
        Alexandre Forel\inst{1}\orcidID{0000-0002-9868-4804} \and
        Axel Parmentier\inst{2}\orcidID{0000-0003-1762-4947} \and
        Thibaut Vidal\inst{1}\orcidID{0000-0001-5183-8485}}
\authorrunning{Forel et al.}

\institute{
Department of Mathematics and Industrial Engineering,\\ Polytechnique Montreal, Montreal, Canada\\ \email{\{alexandre.forel, thibaut.vidal\}@polymtl.ca} \and
CERMICS, Ecole des Ponts, Marne-la-Vall\'ee, France\\
\email{axel.parmentier@enpc.fr}}
\maketitle              

\begin{abstract}
  Counterfactual explanations describe how to modify a feature vector in order to flip the outcome of a trained classifier. Obtaining robust counterfactual explanations is essential to provide valid algorithmic recourse and meaningful explanations. We study the robustness of explanations of randomized ensembles, which are always subject to algorithmic uncertainty even when the training data is fixed. We formalize the generation of robust counterfactual explanations as a probabilistic problem and show the link between the robustness of ensemble models and the robustness of base learners. We develop a practical method with good empirical performance and support it with theoretical guarantees for ensembles of convex base learners. Our results show that existing methods give surprisingly low robustness: the validity of naive counterfactuals is below $50\%$ on most data sets and can fall to $20\%$ on problems with many features. In contrast, our method achieves high robustness with only a small increase in the distance from counterfactual explanations to their initial observations.
\end{abstract}
\section{Introduction}

Counterfactual explanations provide a course of action to change the outcome of a classifier and reach a target class. Since the seminal work of \cite{Wachter2017}, counterfactual explanations have received a large amount of attention \citep{Karimi2022, Verma2020}. Their applications can be divided into two main categories: (1)~to provide algorithmic recourse to users or customers so that they can react to a given outcome (e.g., a customer applies for a loan and is rejected), and (2)~to explain the recommendation of complex classifiers to stakeholders (e.g., a medical diagnosis based on a large set of features).

In high-stakes environments, such as the loan application example, algorithmic recourse must remain valid over time. That is, the loan should be approved when the customer returns after having acted upon the explanation received. Yet, the classification model might be retrained in the meantime, which alters its prediction function. The robustness of explanations when a model is retrained has been studied when additional data is observed, possibly affected by a shift in the distribution of the data-generating process \citep{Bui2022, Dutta2022, Rawal2020, Upadhyay2021}. However, a fundamental source of uncertainty has been overlooked so far: when the classifier is a randomized ensemble, \emph{the random training procedure always leads to algorithmic uncertainty} when retraining the model, even if the training data is fixed.

From an explainability perspective, a lack of robustness due to the random training procedure when the training data is fixed is already very problematic. It raises the question of whether explanations that are not robust to model retraining allow any meaningful interpretation of a classifier's decisions. This issue is akin to the concept of predictive multiplicity, in which models from different classes have similar average performance but wildly different predictions for certain samples \citep{Hsu2022}. By design, nearest counterfactual explanations identify the minimum change that flips the classifier's label. Thus, they might be attracted to regions of the feature space that are particularly vulnerable to predictive multiplicity. In that case, the explanations obtained are only noise: artifacts of the random training procedure that exploits the closest region with high predictive multiplicity. Such non-robust explanations are equivalent to adversarial examples, which have no explainability value but are optimized to fool a given model \citep{Ignatiev2019, Pawelczyk2022}.

In this paper, we study the robustness of counterfactual explanations of randomized ensembles to algorithmic uncertainty. Ensemble learning is a powerful technique that aggregates several models to reduce the risk of over-fitting and achieve high generalization power \citep{Sagi2018, Hastie2009}. Common approaches to building an ensemble of learners revolve around the use of randomization (e.g., the bootstrap aggregating procedure of random forests) and are thus susceptible to algorithmic uncertainty. A key result of our paper is to show that naive explanations that ignore the algorithmic uncertainty of random ensembles are not robust to model retraining even when the training data is fixed. Hence, naive explanations provide neither robust algorithmic recourse nor explainability.

To bridge this gap, we develop methods to obtain counterfactual explanations that are robust to algorithmic uncertainty. We make the following contributions:
\begin{enumerate}[topsep=-0.9\parskip, noitemsep, wide=5pt]
    \item We show that naive methods to generate counterfactual explanations fail to provide robust explanations on common data sets --- the validity is often below $50\%$ and falls below $20\%$ on the most complex data set.
    \item We derive an efficient method to generate robust explanations by identifying a robust threshold on the ensemble's score. Our approach is flexible in the sense that it can be combined with any counterfactual explanation method and applies to all ensembles made of independently trained base learners (e.g., random forests, deep ensembles with random weights initialization). We demonstrate the value of our approach by obtaining explanations of random forest ensembles that are robust to algorithmic uncertainty on real-world data.
    \item We support our practical results with theoretical guarantees that hold for ensembles of convex learners, such as random forests made of trees with a single decision split or ensembles of input-convex neural networks.
    \item Finally, we study the connection between the predictive importance of features and the robustness of counterfactual explanations. We show that generating robust counterfactuals is more challenging for data sets with many features with high predictive importance. 
\end{enumerate}

\section{Problem Statement and Background}
\label{sec:problem}

We consider the standard binary classification setting. A training set $z_n = \left\lbrace(x_i, y_i)\right\rbrace_{i=1}^n$ of size $n$ is available where $x_i \in \mathcal{X}$ is a feature vector and $y_i \in \left\lbrace 0, 1 \right\rbrace$ is a binary label. We assume throughout the paper that the training samples are i.i.d. observations of an unknown distribution $P_{XY}$.

\subsection{Classification Ensembles}

We focus on classification ensembles made of base learners trained independently and identically. This encompasses for instance the random forests of \cite{Breiman2001} or ensembles of neural networks. The randomness of random forests stems from two factors: (1)~each base learner uses a re-sampling of the original training set (most commonly a bootstrap sample), and (2)~a random subset of features is selected at each node of the tree to identify the best split. Interestingly, bootstrapped ensembles of neural networks do not benefit from the performance improvements of their tree-based counterparts. Instead, deep ensembles are often made of base learners trained independently on the same training set with random initial weights \citep{Lakshminarayanan2017}. Randomization improves the generalization performance of ensembles as it reduces the correlation between base learners and thus significantly reduces variance (see Chapter 15, Section 15.4 of \cite{Hastie2009}). 

Formally, we denote by $t(\cdot;\xi) : \mathcal{X} \rightarrow \{0,1\}, x \mapsto t(x;\xi)$ the prediction function of a base learner with random training procedure parameterized by $\xi$. A classification ensemble $T(\cdot;\bm{\xi}) : \mathcal{X} \rightarrow \{0,1\}, x \mapsto T(x;\bm{\xi})$ consists of $N$ base learners $\{t(\cdot; \xi_i)\}_{i=1}^N$, where $\bm{\xi}=\{\xi_i\}_{i=1}^N$ parameterizes the training procedure of the ensemble. We denote by $h_N(\cdot;\bm{\xi}): \mathcal{X} \rightarrow [0,1]$ the score function of an ensemble, defined as the average of the class predictions:
\begin{equation}
    h_N(x;\bm{\xi}) = \frac{1}{N} \sum_{i=1}^N t(x;\xi_i),
\end{equation}
For any observation $x\in\mathcal{X}$, a trained ensemble returns a class prediction as a majority vote among the base learners, so that $T(x;\bm{\xi}) = 1 $ if $h_N(x;\bm{\xi}) \ge 1/2$ and $T(x;\bm{\xi}) = 0 $ otherwise.

\subsection{Counterfactual Explanations of Ensembles}
Let $T^0$ be a classification ensemble trained on $z_n$ with score function $h_N^0$, and let $x_{n+1}$ be a new observation with predicted class $T^0(x_{n+1};\bm{\xi}^0)$. We want to find the nearest counterfactual explanation of this class prediction. We assume, w.l.o.g, that the predicted class of $x_{n+1}$ is $T^0(x_{n+1};\bm{\xi}^0)=0$ so that the target class of its counterfactual explanation is~$1$. Multi-class problems can be converted to this setting by assigning label~$1$ to the target class and label~$0$ to all other classes.

We focus on nearest counterfactual explanations that minimize a distance function $f(\cdot, x_{n+1}) $, such as the $l_1$-norm that encourages sparse explanations for which only a few features are modified. Counterfactual explanations should satisfy two essential conditions: they should be actionable and plausible. Actionability ensures that users can act upon the algorithmic recourse by making sure that immutable features are not modified or respect a specific structure \citep{Ustun2019}. For instance, an individual cannot decrease their age. Plausibility ensures that the counterfactual explanation is not an outlier of the distribution of the target class. This can be done by ensuring that the likelihood of an explanation is larger than a desired threshold, where the likelihood can be estimated for instance using density estimation \citep{Artelt2020}, local outlier factor \citep{Kanamori2020}, or local-neighborhood search \citep{Laugel2020}. An open question is whether plausibility can also help in obtaining explanations that are robust to the random training procedure. We show in our experiments that this is not the case.

The majority of existing counterfactual methods modify the given sample until the target class is attained \citep[see e.g.,][]{Tolomei2017, Lucic2022}. Thus, they work in a heuristic fashion and do not guarantee that the explanation found is the closest one to the original sample. Conversely, approaches based on integer programming can determine counterfactuals that are optimal for the distance metric under consideration and readily integrate constraints that reflect the actionability of the feature changes \citep{Kanamori2020}. Such methods have proved especially relevant for generating counterfactual explanations over tree ensembles \citep{Cui2015}. In particular, \cite{Parmentier2021} provide an efficient formulation for cost-optimal counterfactuals in tree ensembles. Plausibility constraints are integrated using isolation forests, a tree-based method to estimate the likelihood of an explanation.

We can now state the problem of finding the nearest counterfactual explanation of $x_{n+1}$ in a general way as:
\begin{subequations}
	\label{opt:ce}
	\begin{alignat}{2}
		\min_x  & \quad && f(x, x_{n+1}) \label{opt:ce:obj}\\
		\text{s.t.} & && h_N^0(x;\bm{\xi}^0) \ge 1/2 \label{opt:ce:target}\\
		 & && x \in \mathcal{X}^a \cap \mathcal{X}^p . \label{opt:ce:domain}
	\end{alignat}
\end{subequations}
Constraint~\eqref{opt:ce:target} ensures that the counterfactual explanation reaches the target class according to the ensemble $T^0$. Constraint~\eqref{opt:ce:domain} specifies that the counterfactual explanation belongs to both the actionable domain $\mathcal{X}^a$ and the plausible domain $\mathcal{X}^p$. We call counterfactual explanations obtained by solving Problem~\eqref{opt:ce} \textit{naive} since they ignore the algorithmic uncertainty caused by the random training procedure of the ensemble.

\subsection{Algorithmic Uncertainty, Validity and Robustness}

Let $T^0$ and $T^1$ be two classification ensembles of size $N$ trained on the same training set $z_n$. Due to the random training procedure, the classifiers have different prediction functions. Let $\CE (\, \cdot \, ; T^0) : \mathcal{X} \rightarrow \mathcal{X}$ be an algorithm that maps any observation~$x$ to a counterfactual explanation $\hat{x}$ for the classifier $T^0$. We now formalize the concepts of validity and robustness of counterfactual explanations.
\begin{definition}[Validity]
    The counterfactual explanation $\hat{x}=\CE(x;T^0)$ is valid for classifier $T^1$ if $T^1(\hat{x}, \bm{\xi^1}) = 1$.
\end{definition}
\begin{definition}[Algorithmic robustness]
    \label{def:algoRobustness}
    A counterfactual algorithm $\CE \left(\, \cdot \, ; T^0 \right)$ is robust with tolerance $\alpha$ if, for any new observation $x_{n+1}$, the probability that its counterfactual explanation is valid for a new classifier $T$ trained on $z_n$ is greater than $(1-\alpha)$, that is:
    \begin{equation}
        \Prob_{\bm{\xi}} \left( T \left( \CE \left(x_{n+1} ; T^0\right);\bm{\xi} \right) = 1 \right) \ge 1 - \alpha,
    \end{equation}
     where the uncertainty is taken with regard to the random training procedure of $T(\cdot \, ; \bm{\xi})$ on the fixed training set $z_n$.
\end{definition}

Definition~\ref{def:algoRobustness} introduces robustness in a probabilistic sense over all possible ensembles of size $N$ trained on the fixed set~$z_n$. Intuitively, the tolerance parameter $\alpha$ controls the trade-off between the expected robustness of the counterfactual algorithm and the average distance between counterfactual explanations and the original observations. One of the goals of this work is to investigate the trade-off between these two objectives and to provide robust counterfactuals that remain close to their original observations.
\section{Robust Counterfactual Explanations}
\label{sec:method}

Due to the random training procedure of randomized ensembles, the class prediction of an ensemble trained on $z_n$ is a random variable. To generate a counterfactual explanation robust to algorithmic uncertainty, Problem~\eqref{opt:ce} has to be augmented with the probabilistic constraint introduced in Definition~\ref{def:algoRobustness}, which can be equivalently expressed as:
\begin{equation}
	\label{eq:probCons}
	\Prob_{\bm{\xi}} \left( h_N(x;\bm{\xi}) \ge 1/2 \right) \ge 1-\alpha.
\end{equation}

\textbf{Main result.} The core result of our paper is to show that, given a trained ensemble with score function $h_N^0$, we can obtain robust counterfactual explanations by replacing the complex probabilistic condition in Equation~\eqref{eq:probCons} with the much simpler deterministic condition:
\begin{equation}
	\label{eq:robCe}
    h_N^0(x;\bm{\xi}^0) \ge \tau(N, \alpha),
\end{equation}
where $\tau(N, \alpha) \in [1/2, 1]$ is a well-defined threshold. The computational complexity of finding robust explanations is thus the same as the one of the naive Problem~\eqref{opt:ce}.

Deriving this result consists of two steps: (i)~reformulating the robustness constraint on ensembles into an equivalent constraint on base learners, and (ii)~showing that the probability of any base learner to output the target class is well approximated by the score function of the initial ensemble. The remainder of this section details these two steps.

\subsection{Reformulating the Robustness Constraint}
\label{sec:treeRob}

We start by formalizing the statistical properties of the class predicted by a randomized ensemble when the training data $z_n$ is fixed. For any observation $x \in \mathcal{X}$, the event that a base learner trained on $z_n$ outputs the target class at $x$ can be seen as a random event with probability $p(x)$. Since we consider a binary classification setting, the class predicted by a base learner trained on $z_n$ thus follows a Bernoulli distribution: $t(x;\xi) \sim \Ber(p(x))$.

A key observation is that, given a randomized ensemble trained on $z_n$, its base learners are independent and identically distributed observations of any base learner trained on $z_n$. Consequently, the score of a randomized ensemble with $N$ base learners follows a binomial distribution:
\begin{equation}
    N\cdot h_N(x, \bm{\xi}) \sim \Binom(N, p(x)).
\end{equation}

Denote by $B(k;N,p)$ the cumulative distribution  function (c.d.f.) of the binomial distribution $\Binom(N, p)$ evaluated at $k$. The following property can be identified.
\begin{lemma}
    \label{lem:gN}
    Given $N \in \mathbb{N}$, the map $g_N: [0, 1] \to [0, 1], p \mapsto B\left(N/2 ; N, p\right)$ is decreasing and invertible.
\end{lemma}
The proof is based on the link between the c.d.f. of the binomial distribution and the c.d.f. of the beta distribution. It is given in Appendix~\ref{app:gNproof}
. Using Lemma~\ref{lem:gN}, we can reformulate the probabilistic condition in Equation~\eqref{eq:probCons} by an equivalent condition on the prediction of a single (random) base learner.
\begin{proposition}
    \label{th:fromForestToTree}
    Let $h_N$ be the score function of an ensemble of $N$ base learners trained on~$z_n$. The robustness condition $\Prob_{\bm{\xi}} \left( h_N(x, \bm{\xi}) \ge 1/2 \right) \ge (1-\alpha)$ is satisfied if $\Prob_{\xi} \left( t(x, \xi) = 1 \right) \ge g_N^{-1}(\alpha)$. The two conditions are equivalent when $N$ is odd.
\end{proposition}
 The proof is given in Appendix~\ref{app:fromForestToTree}. Proposition~\ref{th:fromForestToTree} provides a criterion to generate robust counterfactual explanations based only on the probability that a base learner predicts the target class. We denote by $\pStar = g_N^{-1}(\alpha)$ this key robustness threshold and provide two properties. They describe how the robustness threshold $\pStar$ varies with the tolerance $\alpha$ and the ensemble size $N$.
\begin{proposition}
    \label{prop:alphaSens}
    Let $N \in \mathbb{N}$, the threshold $\pStar$ is monotonic increasing with the robustness target $(1-\alpha)$.
\end{proposition}
The proof follows directly from Lemma~\ref{lem:gN}. This proposition confirms the intuitive property that a counterfactual explanation algorithm that is robust with a robustness level $(1-\alpha)$ is also robust at a robustness level lower than $(1-\alpha)$. We now study how increasing the number of base learners impacts robustness. We will need the following lemma.
\begin{lemma}
    \label{lem:p_var_with_N}
    Let $m \in \mathbb{N}$, the following relationships hold:
    \begin{itemize}[topsep=-0.9\parskip,
                    noitemsep, wide=5pt]
        \item[] (a) \ $\forall \alpha \le 1/2, \,  p_{2m+3, \alpha}^* \le p_{2m+1, \alpha}^*$,
        \item[] (b) \ $\forall \alpha \le 1/2, \, p_{2(m+1), \alpha}^* \le p_{2m, \alpha}^*$,
        \item[] (c) \ $\forall \alpha, \, p_{2m+1, \alpha}^* \le p_{2m, \alpha}^*$.
    \end{itemize}
\end{lemma}
The proof is based on studying the variations of $B(k;N,p)$ with fixed $p$ and is given in Appendix~\ref{app:proof_of_p_var_with_N}.
\begin{proposition}
    \label{prop:nSens}
    A counterfactual algorithm robust to algorithmic uncertainty for an ensemble of size $N$ with $N$ even is also robust for any ensemble of larger size.
\end{proposition}
The proof follows directly from Lemma~\ref{lem:p_var_with_N}. Proposition~\ref{prop:nSens} generalizes the robustness of counterfactual explanations for an ensemble of size $N$. We illustrate the results of Proposition~\ref{prop:alphaSens} and \ref{prop:nSens} in Figure~\ref{fig:pStarSens}, which shows how the robustness threshold $\smash{p_{N,\alpha}^*}$ varies as a function of the robustness target $(1-\alpha)$ and the ensemble size~$N$.
\begin{figure}[ht!]
    \vspace{-5mm}
    \centering
    \resizebox{0.9\linewidth}{!}{\begin{tikzpicture}
	\begin{groupplot}[
		group style={
			group name=my plots,
			group size=2 by 1,
			xlabels at=edge bottom,
			ylabels at=edge left
		},
		height=5cm,
		width = 7cm,
		ylabel = {Robustness threshold $p_{N, \alpha}^*$},
		legend cell align={left}
		]
		
		\nextgroupplot[title = {(a)}, font = \small,
                        mark size={1},
                	    xmin = 0.5, xmax = 1.0,
                		ymin = 0.5, ymax = 0.675,
		                legend pos= north west,
                		xlabel={Robustness target $(1-\alpha)$},
                		mark repeat = 3,
                		mark phase = 2]
		
		\addplot+[blue, mark=*, mark options={fill=blue}] table [x index = {0}, y index = {1}, col sep=comma]{plots/csv/p_star_sensitivity.csv};
		\addlegendentry{$N=50$}
		
		\addplot+[red, mark=square*, mark options={fill=red}] table [x index = {0}, y index = {2}, col sep=comma]{plots/csv/p_star_sensitivity.csv};
		\addlegendentry{$N=100$}
		
		\addplot+[darkgreen, mark=triangle*, mark options={fill=darkgreen}] table [x index = {0}, y index = {3}, col sep=comma]{plots/csv/p_star_sensitivity.csv};
		\addlegendentry{$N=200$}
		
		\nextgroupplot[title = {(b)}, font = \small,
                        mark size={1},
                	    xmin = 1, xmax = 50,
                	    ymin = 0.2, ymax = 1,
		                legend pos= north east,
			            legend columns=1,
                		xlabel={Ensemble size $N$}]
		\addplot+[blue, mark=*, mark options={fill=blue}] table [x index = {0}, y index = {1}, col sep=comma]{plots/csv/p_star_sensitivity_N.csv};
		\addlegendentry{$\alpha=0.1$}
		\addplot+[red, mark=square*, mark options={fill=red}] table [x index = {0}, y index = {2}, col sep=comma]{plots/csv/p_star_sensitivity_N.csv};
		\addlegendentry{$\alpha=0.25$}
		\addplot+[darkgreen, mark=triangle*, mark options={fill=darkgreen}] table [x index = {0}, y index = {3}, col sep=comma]{plots/csv/p_star_sensitivity_N.csv};
		\addlegendentry{$\alpha=0.5$}
		\addplot+[black, mark=diamond*, mark options={fill=black}] table [x index = {0}, y index = {4}, col sep=comma]{plots/csv/p_star_sensitivity_N.csv};
		\addlegendentry{$\alpha=0.75$}
	\end{groupplot}
\end{tikzpicture}}
    \caption{Sensitivity of the robustness threshold $\pStar$.}
    \label{fig:pStarSens}
    \vspace{-3mm}
\end{figure}
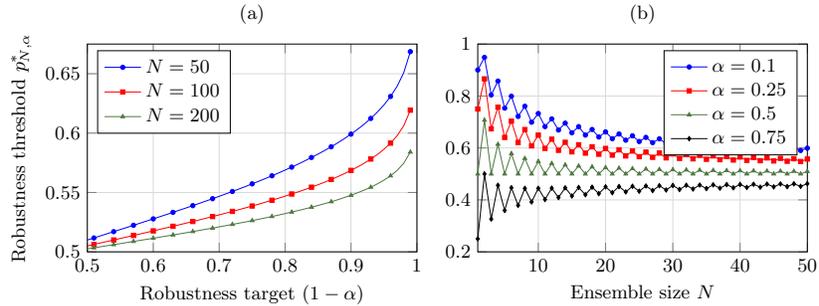

All the above results hold as long as the ensembles are made of base learners trained independently and identically on the fixed data set $z_n$. For instance, they apply to random forests, deep ensembles, or any ensemble built using the bootstrap aggregating procedure. Another observation is that $\alpha = 1/2$ recovers the naive condition presented in Problem~\eqref{opt:ce} when $N$ is odd. Thus, explanations satisfying exactly the naive constraint \eqref{opt:ce:target} have an average validity of only $50\%$ when $N$ is odd (regardless of the ensemble size!), or when $N$ is even and the ensemble size is large as illustrated in Figure~\ref{fig:pStarSens}. Therefore, even if the ensemble size is large, naive algorithms should fail to provide explanations robust to the noise of the training procedure.

\subsection{Sample-Average Approximations}

The robustness condition in Proposition~\ref{th:fromForestToTree} remains a probabilistic constraint and, as such, cannot be directly integrated into any solution algorithm. Further, the probability $\Prob_{\xi} \left( t(x, \xi) = 1 \right)$ that a base learner trained on $z_n$ predicts the target class at $x$ is unknown. We now present two methods to approximate the probability $\Prob_{\xi} \left( t(x, \xi) = 1 \right)$ that use only a given trained ensemble $T^0$.

\textbf{Background.} Optimization problems with a probabilistic constraint such as Equation~\eqref{eq:robCe} belong to the class of chance-constrained problems. These problems have been studied extensively in the stochastic optimization literature since the seminal work of \cite{Charnes1963}. In particular, sampling approaches, such as the sample-average approximation (SAA), are well-known techniques to solve chance-constrained problems with strong theoretical foundations and good empirical results \citep{Pagnoncelli2009}.

To approximate the probability $\Prob_{\xi} \left( t(x, \xi) = 1 \right)$, a possible sampling-based approach is to repeatedly train base learners on $z_n$ and evaluate them for a given observation $x$. Yet, in counterfactual explanation applications, we already have "sampled" observations of the base learner: given the trained ensemble $T^0$, the base learners $\{t^0(x, \xi_i^0)\}_{i=1}^N$ are i.i.d. observations of $t(x, \xi)$. Thus, we do not need to train any additional learners. This is very desirable in practice since it means that robust explanations can be obtained without having access to the training data when a trained ensemble is available.

\textbf{Direct SAA.} Given the trained learners $\{t^0(x, \xi_i^0)\}_{i=1}^N$, the Direct SAA of the probabilistic robustness constraint is:
\begin{equation}
    \label{eq:direct-SAA}
    h_N^0(x, \bm{\xi}^0) \ge \pStar.
\end{equation}
Thus, we recover the simple condition presented in Equation~\eqref{eq:robCe} by taking $\tau(N, \alpha) = \pStar$. Note also that, when $\pStar = 1/2$, we recover the naive condition of Problem~\eqref{opt:ce}.

Sample-average approximations of chance constraints have been shown to give good performance in numerous applications. However, the Direct SAA does not guarantee that the probabilistic constraint is satisfied in general with a finite ensemble size $N$. Hence, we provide a more robust approach, that still uses only the initial ensemble and does not need to train any additional learners. 

\textbf{Robust SAA.} This approach is motivated by a statistical perspective: the score function of a given ensemble $h^0_N$ can be seen as an estimator of the probability that a single learner predicts the target class. Indeed, when binomial samples are observed i.i.d, the sample mean is an unbiased minimum-variance estimator of the success rate $p$ of the underlying Bernoulli distribution. To hedge against the noise of this estimator when the ensemble size is finite, we introduce the Robust SAA based on building a confidence interval around the estimated threshold $\pStar$ as if it were estimated from i.i.d. observations.

Confidence intervals of the success rate of binomial distributions have been studied extensively. In particular, the Agresti-Coull (AC) confidence intervals \citep{Agresti1998} achieve good coverage of the true success rate in finite samples \citep{Brown2001}. The Robust SAA thus uses the threshold:
\begin{equation}
    \label{eq:robust-SAA}
    h_N^0(x, \bm{\xi}^0) \ge \rho_{N, \alpha, \beta}^*,
\end{equation}
where $\rho_{N, \alpha, \beta}^* = \pStar + z_{\beta} \sqrt{\rho_{AC}(1-\rho_{AC})/N}$ with $z_{\beta}$ being the quantile of the standard normal distribution at $\beta/2$, and $\rho_{AC}= (N \cdot \pStar + 2)/(N+4)$. The confidence level $\beta \in [0, 1]$ is a hyperparameter that adjusts the conservativeness of the solution. As $\beta$ increases, the robustness of the counterfactual increases and so does the distance to the initial observation. Thus, $\beta$ depends on the data set and needs to be tuned according to the desired robustness level. This approach is always more conservative than the Direct SAA since $\beta=0$ recovers the Direct SAA.
\section{Robustness Guarantees for Ensembles of Convex Learners}
\label{sec:stumpRob}

Our approach can be supported by statistical guarantees by leveraging the theory on sample-average approximations developed in the stochastic optimization literature. In particular, when the ensemble is made of convex base learners, the Direct SAA is asymptotically consistent.

\begin{proposition}[Asymptotic consistency]
    \label{prop:asympt}
    As the size of an ensemble of convex base learners increases, the solution of the Direct SAA method converges almost surely to the minimum-cost robust counterfactual explanation.
\end{proposition}
When the base learners are convex functions of the features, the condition $t(x;\xi) = 1$ inside the probabilistic constraint of Proposition~\ref{th:fromForestToTree} is also convex. Hence, the proof follows from \cite[Chapter 5]{Shapiro2014}. Proposition~\ref{prop:asympt} holds for several classes of randomized ensembles. By definition, it holds for ensembles made of input-convex neural networks \citep{Amos2017}.

We can show that the result also holds for random forests made of decision trees that use a single decision split (also called stumps). These simplified forests have been studied for instance by \cite{Buhlmann2002} to show how bagging reduces variance in random forests. The class prediction of tree stumps can be formulated as a convex constraint for any realization of the uncertain training procedure. Let $\mathcal{X} \subseteq [0, 1]^d$ and $t(x, \xi)$ be a decision stump. The robustness constraint can be expressed as a convex constraint as:
\begin{equation}
    \label{eq:stumpRob}
    \Prob (t(x, \xi) = 1) = \Prob \left(A(\xi)^{\top}x - b(\xi) \le 0 \right),
\end{equation}
where $A(\xi) \in \{-1, 0, 1\}^d$ and $b(\xi) \in [-1, 1]$. The vector $A(\xi)$ is such that $a_j(\xi) = 0$ if the stump does not split on feature $j$, $a_j(\xi) = 1$ if it splits on feature $j$ and the left leaf node has class $1$ and $a_j(\xi) = -1$ if it splits on feature $j$ and the right leaf node has class $1$. The split threshold $b(\xi)$ is positive if the left leaf node has class $1$ and negative otherwise. Thus, the function on the left-hand side of the probabilistic Constraint~\eqref{eq:stumpRob} is convex in $x$.

These results on asymptotic consistency are notable since they hold for a wide variety of randomized ensembles. They do not require any simplification of the training procedure of the base learners, as is common for instance when studying the theoretical properties of random forests \citep{Biau2016}. Finally, it is interesting to observe that increasing the size of the ensemble has two effects: (1)~it decreases the robustness threshold $\pStar$ as shown in Lemma~\ref{lem:p_var_with_N}, and (2)~it increases the accuracy of the Direct-SAA method according to Proposition~\ref{prop:asympt}.

\textbf{Finite-sample guarantees.} In certain cases, it is not possible to train additional learners when determining counterfactual explanations, such as when only the initial ensemble $T^0$ is available. In this case, we are interested in finite-sample guarantees on the robustness of explanations. Finite-sample bounds on the quality of the Direct-SAA method are given by \cite{Luedtke2008}, but require stringent assumptions on the feature space $\mathcal{X}$ and do not hold if the feature vector contains both continuous and discrete features. We present finite-sample bounds based on a second approximation technique called the convex approximation.

The convex approximation of the probabilistic robustness condition in Equation~\eqref{eq:probCons} results in the following set of constraints:
\begin{equation}
    \label{eq:convexApprox}
    t^0(x, \xi_i^0) = 1, \forall i \in \{1, \dots, N\}.
\end{equation}
Thus, the convex approximation is equivalent to taking $\tau(N, \alpha) = 1$ in Equation~\eqref{eq:robCe}. Finite-sample bounds on the probability of finding a robust solution by solving the convex approximated model can be obtained. \cite{Campi2008} provide a key result when the feasible set $\mathcal{X}$ is convex (for instance, when all features are continuous). \cite{DeLoera2018} generalize it to decision variables that take both continuous and discrete values. We can apply the latter result to obtain the following bound.
\begin{proposition}[Finite-sample guarantees]
    Given an ensemble of convex learners of size $N$ and a feature vector $x$ with $k$ continuous features and $d-k$ discrete features, if a solution to the convex approximated problem exists, the probability that it is a robust counterfactual explanation is at least (1-$\delta$) with $\delta =\exp \left[\left(2^{d-k}(k+1)-1 \right)\left(\log(1/\alpha) + \alpha\right) - (\alpha/2)N \right]$.
\end{proposition}
The proof follows directly from \cite{DeLoera2018} when the base learners are convex in $x$. Note that $\delta$ decreases exponentially as the number of trees increases.
\section{Experimental Results}
\label{sec:numStudy}
We conduct extensive experiments to (i)~evaluate the robustness of our proposed approaches, (ii)~understand why counterfactual explanations have varying robustness on different data sets, and (iii)~demonstrate that our methods provide the best trade-off between robustness and distance. We focus our experiments on random forests, which are arguably the most common randomized ensembles used in practice and remain one of the state-of-the-art classifiers for tabular data \citep{Grinsztajn2022, Biau2016}. We generate counterfactual explanations using the state-of-the-art formulation of \cite{Parmentier2021} based on integer programming. An integer programming method has the advantage of being exact: it returns the counterfactual explanation with minimal distance to the original observation. Hence, it eliminates confounding factors in the analyses due to the choice of a specific heuristic solution.

The simulations are implemented in Python 3.9 using scikit-learn v.1.0.2 to train the random forests. Gurobi 9.5 is used to solve all integer programming models. All experiments are run on an Intel(R) Core(TM) i7-11800H processor at $2.30$Ghz using $16$GB of RAM. The data sets have been pre-processed following \cite{Parmentier2021} to ignore missing values and to take into account feature actionability. A summary of the data sets considered is provided in Appendix~\ref{app:datasets}. The code to reproduce all results and figures in this paper is publicly available at the online repository \url{https://github.com/alexforel/RobustCF4RF} under an MIT license.

\textbf{Simulation setting.} We use the standard procedure of scikit-learn to generate random forests of $N=100$ trees (default value of scikit-learn) with a maximum depth of $4$. In each simulation, five samples are selected randomly to serve as new observations for which to derive counterfactuals. A first forest $T^0$ is trained on the remaining $(n-5)$ points. A second forest $T^1$ is trained on the same data to asses if the counterfactual explanations are valid. We repeat this procedure forty times.

We implement the Direct-SAA and Robust-SAA methods with $\beta \in \{0.05, 0.1\}$ and vary the tolerance parameter $\alpha$ between $0.5$ and $0.01$. The key performance indicators are the distance between the initial observation and the counterfactual and their validity, measured as the percentage of counterfactuals that are valid to the test classifier. All distances are measured following a feature-weighted $l_1$ norm. We use the $l_1$ norm since it tends to create sparse explanations, that is, explanations with a limited number of changed features. To balance the cost of feature changes between continuous and non-continuous features, we reduce the weight of changes on non-continuous features by a factor~$1/4$.

\textbf{Benchmarks.}
Existing works on counterfactual explanations do not study the algorithmic uncertainty caused by a randomized training procedure. Hence, there is no directly related benchmark. The closest approach from the existing literature to our setting is arguably the one from \cite{Dutta2022}. Their algorithm (RobX) aims to find explanations of tree ensembles that are robust to evolving data sets and hyperparameters. Even though they study a different type of robustness, we evaluate the robustness of their explanations to algorithmic uncertainty. RobX is based on iteratively perturbating a naive explanation by moving it toward a neighbor with high stability, in the sense that the prediction model is locally constant. Since it is unclear how this method can be applied with binary, discrete, or categorical features, we only implement it for the \textsc{Spambase} dataset, which has only continuous features.

We further include three benchmarks: (1)~the naive approach that uses the constraint $h_N(x) \ge 1/2$, (2)~a plausibility-based benchmark that uses isolation forests \citep{Parmentier2021, Liu2008}, and (3)~a plausibility-based benchmark that uses the local outlier factor as \cite{Kanamori2020}. Plausibility methods encourage the explanation to be close to the training data distribution. We include these benchmarks to investigate whether producing explanations that are close to the data distribution is sufficient to ensure robustness to algorithmic uncertainty. The contamination parameter of the isolation forest method is varied as $c \in \{0.05, 0.1, 0.2, 0.3, 0.4, 0.5\}$. Similarly, the weight of the local outlier factor penalty term is varied as $\lambda \in \{1e^{-3}, 1e^{-2}, 1e^{-1}, 1, 1e^{1}, 1e^{2}\}$ following \cite{Kanamori2020}. These two hyperparameters control the degree to which the obtained explanations are close to the training data. Details on the implementation of our methods and benchmarks are provided in Appendix~\ref{app:implementation}, which also shows the computation time of the different methods.

\textbf{Example of explanations with increasing robustness.}
We illustrate the generated counterfactual explanations with varying robustness targets for the \textsc{German Credit} data set in Figure~\ref{fig:cfTraj}. The feature values of an initial observation $x_{n+1}$ are shown on the bottom row as a heatmap. Each row then shows the changes to the initial feature values as the target robustness level $(1-\alpha)$ increases. Positive changes to a feature are shown in blue and negative changes are shown in red. In this example, the ''Age'' feature can only increase. The row with $(1-\alpha)=0.5$ thus shows a naive explanation, whereas the top row with $(1-\alpha) = 0.99$ shows an explanation with high robustness.

Figure~\ref{fig:cfTraj} illustrates how the number and magnitude of feature changes vary as $(1-\alpha)$ increases. It shows that a user can obtain a robust explanation with only a small subset of features changed. Additional examples are presented in Appendix~\ref{app:cfTraj}.
\begin{figure}[ht]
    \centering
    \vspace{-2mm}
    \includegraphics[width=0.55\linewidth]{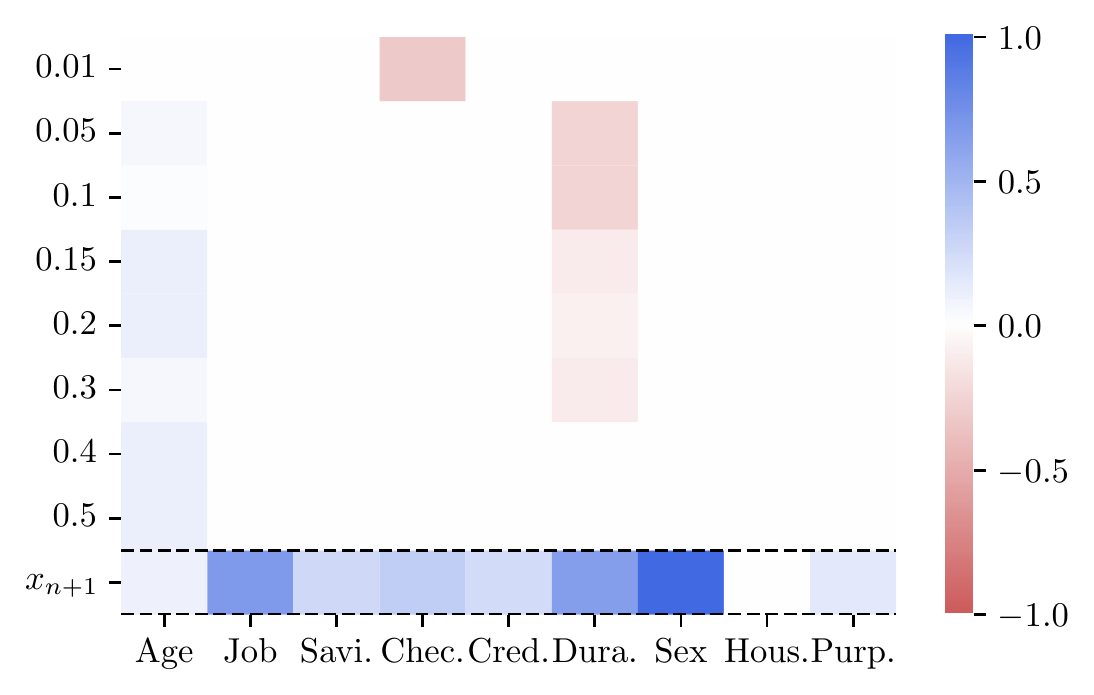}
    \caption{Initial observation and counterfactual explanations for increasing robustness target ($1-\alpha$).}
    \label{fig:cfTraj}
    \vspace{-5mm}
\end{figure}

\subsection{Achieving Robust Counterfactual Explanations}
\label{sec:robExp}

This section provides our main experimental results. Additional experiments, using larger random forests or an evolving data set, are presented in Appendix~\ref{app:addExp}.

\textbf{Naive explanations.} First, we measure the robustness of naive explanations that ignore algorithmic uncertainty. We present the average validity of the counterfactual explanations in Table~\ref{tab:naive}. Although it depends on the data set, naive explanations are clearly not robust to algorithmic uncertainty. The average validity even falls below $20\%$ for the \textsc{Spambase} dataset.
\begin{table}[ht!]
    \vspace{-4mm}
    \center
    \caption{Average naive explanations that remain valid when retraining the ensemble with fixed training data.}
    \label{tab:naive}
    \begin{tabular}{lcccccccc}
    \toprule
    \textbf{Data set} &   A  &  C  &   CC  &  GC  &  ON  &  P  &  S   &   SP \\
    \midrule
    Validity [in \%]  &  62 &  92 &  27 &  35 &  32 &  80 &  17 &  39 \\
    \bottomrule
\end{tabular}
    \vspace{-4mm}
\end{table}

\textbf{Algorithmic robustness.} The validity of the counterfactual explanations generated by the Direct- and Robust-SAA methods is shown in Figure~\ref{fig:algRobResults-validity} for varying robustness target levels. The confidence interval of the average validity at the $0.05$ level is shown as a shaded area. Figure~\ref{fig:algRobResults-validity} shows that the Direct-SAA method provides robust counterfactual explanations on all but one data sets. Counterfactual explanations with high robustness can already be found for small robustness targets in two data sets: \textsc{Compas} and \textsc{Phishing}. Conversely, the data set \textsc{Spambase} with $d=57$ continuous features proves the most difficult. It is only on this data set that the Robust-SAA method is required, yielding robust counterfactuals for moderate and high robustness levels with $\beta=0.1$ and $\beta=0.05$, respectively.
\begin{figure*}[th]
    \vspace{-1mm}
    \centering
    \resizebox{\linewidth}{!}{\input{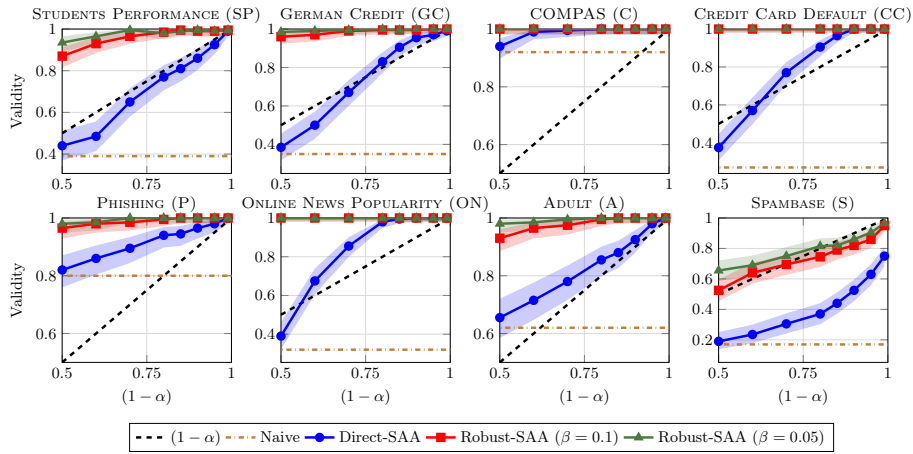}}
    \caption{Validity of robust counterfactuals as a function of the robustness target $(1-\alpha)$.}
    \label{fig:algRobResults-validity}
    \vspace{-3mm}
\end{figure*}

The same experiment is performed with random forests of size $N=400$. The results, presented in Appendix~\ref{app:largeForests}, are essentially identical to those in Figure~\ref{fig:algRobResults-validity}. This confirms that increasing the size of the ensemble does not protect against algorithmic uncertainty if the counterfactual explanation threshold is not chosen appropriately. We also perform experiments when additional data is observed before re-training the random forest. This experiment is presented in Appendix~\ref{app:evolvingData}. Again, the results are very similar to those in Figure~\ref{fig:algRobResults-validity} and our approaches provide robust explanations on all data sets. This suggests that, for randomized ensembles, the algorithmic uncertainty due to the random training procedure is more critical for the robustness of explanations than the uncertainty caused by observing additional data samples.

\textbf{Trade-off between distance and robustness.} The distance and validity of the counterfactual explanations obtained using our methods and benchmarks are shown in Figure~\ref{fig:pareto} as a Pareto front. The results show that our methods strictly dominate the plausibility-based benchmarks by consistently providing more robust explanations with lower distance. This means that plausibility and robustness are independent in practice. Even with high contamination parameters, isolation forests do not provide robust explanations although they substantially increase counterfactual distance. Increasing the penalty factor of the local outlier factor (1-LOF) method slightly increases robustness but leads to significantly more distant counterfactual explanations than the ones generated by our robust methods. The low robustness of the naive and plausibility-based benchmarks suggests their counterfactuals are mostly fitting the noise of the random training procedure. The RobX benchmark also does not provide robust explanations. This suggests that the algorithmic uncertainty of model retraining is distinct from the uncertainty of an evolving data set, which is the focus of RobX. Thus, these methods provide neither robust algorithmic recourse nor meaningful explanations of the random forest.
\begin{figure*}[th]
    \vspace{-1mm}
    \centering
    \resizebox{\linewidth}{!}{\input{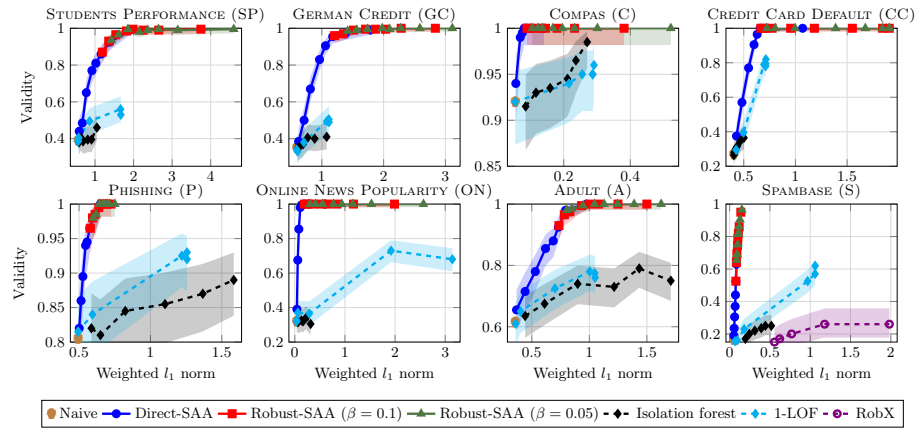}}
    \caption{Trade-off between the distance and robustness of counterfactual explanations.}
    \label{fig:pareto}
    \vspace{-3mm}
\end{figure*}

Figure~\ref{fig:pareto} also highlights that choosing a value for $\alpha$ depends on the preference of the decision-maker. Robust counterfactuals provide more insightful explanations but may be more difficult to act upon due to the increase in the average distance. Still, our results show that a large increase in robustness can be obtained for only a small increase in the counterfactual distance. A satisfying compromise can be found for instance using $\alpha = 0.2$.

\subsection{Feature Importance and Robustness}
To explain the varying robustness achieved on the different data sets, we analyze two aspects: (i)~the average number of features changed in counterfactual explanations, and (ii)~the link between the changed features and their predictive importance. Figure~\ref{fig:featureChanges} shows the average number of changed features as a function of the robustness target level. Data sets that have a few key features with high importance tend to exhibit sparse and robust counterfactual explanations. This is especially true if the features are discrete or categorical as is the case for the \textsc{Compas} data set for instance. This sparsity leads to robust counterfactual explanations even when the robustness target level is low. On the contrary, data sets that have many important features have low inherent robustness, especially when these are continuous features as in the case of the \textsc{Spambase} dataset.
\begin{figure}[ht]
    \vspace{-3mm}
    \centering
    \resizebox{0.6\linewidth}{!}{\begin{tikzpicture}
	\begin{axis}[
		height=4.5cm,
		width = 8.5cm,
		xmin = 0.5, xmax = 1,
		ymin = 1, ymax = 15,
		xlabel={Robustness level $(1-\alpha)$ [in $\%$]},
		ylabel = {Nb. of features changed},
		legend cell align={left},
		ytick = {1, 5, 10, 15},
		legend style={at={(1.25,1)},anchor=north east},
		legend columns=1]
		
		\addplot table [x index = {0}, y index = {1}, col sep=comma]{plots/csv/featureChangeWithAlpha.csv};
		\addlegendentry{SP}
		\addplot table [x index = {0}, y index = {2}, col sep=comma]{plots/csv/featureChangeWithAlpha.csv};
		\addlegendentry{GC}
		\addplot+[darkgreen, mark options={fill=darkgreen}] table [x index = {0}, y index = {4}, col sep=comma]{plots/csv/featureChangeWithAlpha.csv};
		\addlegendentry{C}
		\addplot table [x index = {0}, y index = {5}, col sep=comma]{plots/csv/featureChangeWithAlpha.csv};
		\addlegendentry{P}
		\addplot table [x index = {0}, y index = {6}, col sep=comma]{plots/csv/featureChangeWithAlpha.csv};
		\addlegendentry{CC}
		\addplot table [x index = {0}, y index = {7}, col sep=comma]{plots/csv/featureChangeWithAlpha.csv};
		\addlegendentry{ON}
		\addplot+[darkgreen, mark options={fill=darkgreen}] table [x index = {0}, y index = {8}, col sep=comma]{plots/csv/featureChangeWithAlpha.csv};
		\addlegendentry{A}
		\addplot table [x index = {0}, y index = {3}, col sep=comma]{plots/csv/featureChangeWithAlpha.csv};
		\addlegendentry{S}
	\end{axis}
\end{tikzpicture}}
    \caption{Average number of features changed for varying robustness targets $(1-\alpha)$.}
    \label{fig:featureChanges}
    \vspace{-3mm}
\end{figure}
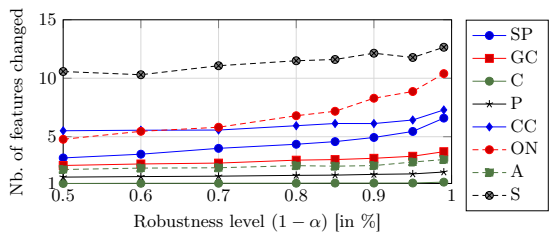

The analysis of feature importance is provided in Appendix~\ref{app:featImportanceAndSparsity}. It shows that feature changes in counterfactual explanations follow largely their predictive importance: highly predictive features have large changes, whereas unimportant features are mostly unchanged.
\section{Conclusions}
\label{sec:conclusion}
This work sheds light on a critical, yet overlooked, aspect of interpretability in machine learning: the lack of robustness of existing counterfactual explanation methods with regard to algorithmic uncertainty. We contribute to building more transparent classifiers by providing an effective and easy-to-implement approach to improve robustness. Our method is supported by a probabilistic analysis of randomized ensembles and its value is demonstrated on real-world data. To help practitioners identify situations in which robust counterfactual explanations are essential, we derive valuable insights linking the robustness of counterfactual explanations and the characteristics of the data sets.

%
%
\bibliographystyle{apalike}
\bibliography{bibliography}

\begin{thebibliography}{}

\bibitem[Abramowitz et~al., 1988]{Abramowitz1988}
Abramowitz, M., Stegun, I.~A., and Romer, R.~H. (1988).
\newblock {\em Handbook of mathematical functions with formulas, graphs, and
  mathematical tables}.
\newblock American Association of Physics Teachers.

\bibitem[Agresti and Coull, 1998]{Agresti1998}
Agresti, A. and Coull, B.~A. (1998).
\newblock Approximate is better than “exact” for interval estimation of
  binomial proportions.
\newblock {\em The American Statistician}, 52(2):119--126.

\bibitem[Amos et~al., 2017]{Amos2017}
Amos, B., Xu, L., and Kolter, J.~Z. (2017).
\newblock Input convex neural networks.
\newblock In {\em International Conference on Machine Learning}, pages
  146--155. PMLR.

\bibitem[Artelt and Hammer, 2020]{Artelt2020}
Artelt, A. and Hammer, B. (2020).
\newblock Convex density constraints for computing plausible counterfactual
  explanations.
\newblock In {\em International Conference on Artificial Neural Networks},
  pages 353--365. Springer.

\bibitem[Biau and Scornet, 2016]{Biau2016}
Biau, G. and Scornet, E. (2016).
\newblock A random forest guided tour.
\newblock {\em Test}, 25(2):197--227.

\bibitem[Breiman, 2001]{Breiman2001}
Breiman, L. (2001).
\newblock Random forests.
\newblock {\em Machine Learning}, 45(1):5--32.

\bibitem[Brown et~al., 2001]{Brown2001}
Brown, L.~D., Cai, T.~T., and DasGupta, A. (2001).
\newblock Interval estimation for a binomial proportion.
\newblock {\em Statistical Science}, 16(2):101--133.

\bibitem[B{\"u}hlmann and Yu, 2002]{Buhlmann2002}
B{\"u}hlmann, P. and Yu, B. (2002).
\newblock Analyzing bagging.
\newblock {\em The Annals of Statistics}, 30(4):927--961.

\bibitem[Bui et~al., 2022]{Bui2022}
Bui, N., Nguyen, D., and Nguyen, V.~A. (2022).
\newblock Counterfactual plans under distributional ambiguity.
\newblock In {\em International Conference on Learning Representations}.

\bibitem[Campi and Garatti, 2008]{Campi2008}
Campi, M.~C. and Garatti, S. (2008).
\newblock The exact feasibility of randomized solutions of uncertain convex
  programs.
\newblock {\em SIAM Journal on Optimization}, 19(3):1211--1230.

\bibitem[Charnes and Cooper, 1963]{Charnes1963}
Charnes, A. and Cooper, W.~W. (1963).
\newblock Deterministic equivalents for optimizing and satisficing under chance
  constraints.
\newblock {\em Operations Research}, 11(1):18--39.

\bibitem[Cui et~al., 2015]{Cui2015}
Cui, Z., Chen, W., He, Y., and Chen, Y. (2015).
\newblock Optimal action extraction for random forests and boosted trees.
\newblock In {\em Proceedings of the 21th ACM SIGKDD International Conference
  on Knowledge Discovery and Data Mining}, pages 179--188.

\bibitem[De~Loera et~al., 2018]{DeLoera2018}
De~Loera, J.~A., La~Haye, R.~N., Oliveros, D., and Rold{\'a}n-Pensado, E.
  (2018).
\newblock Chance-constrained convex mixed-integer optimization and beyond: Two
  sampling algorithms within {$S$}-optimization.
\newblock {\em Journal of Convex Analysis}, 25(1):201--218.

\bibitem[Dutta et~al., 2022]{Dutta2022}
Dutta, S., Long, J., Mishra, S., Tilli, C., and Magazzeni, D. (2022).
\newblock Robust counterfactual explanations for tree-based ensembles.
\newblock In {\em International Conference on Machine Learning}, pages
  5742--5756. PMLR.

\bibitem[Grinsztajn et~al., 2022]{Grinsztajn2022}
Grinsztajn, L., Oyallon, E., and Varoquaux, G. (2022).
\newblock Why do tree-based models still outperform deep learning on typical
  tabular data?
\newblock In {\em Thirty-sixth Conference on Neural Information Processing
  Systems Datasets and Benchmarks Track}.

\bibitem[Hastie et~al., 2009]{Hastie2009}
Hastie, T., Tibshirani, R., and Friedman, J.~H. (2009).
\newblock {\em The Elements of Statistical Learning: Data Mining, Inference,
  and Prediction}, volume~2.
\newblock Springer.

\bibitem[Hsu and Calmon, 2022]{Hsu2022}
Hsu, H. and Calmon, F. d.~P. (2022).
\newblock Rashomon capacity: A metric for predictive multiplicity in
  classification.
\newblock In {\em Advances in Neural Information Processing Systems}.

\bibitem[Ignatiev et~al., 2019]{Ignatiev2019}
Ignatiev, A., Narodytska, N., and Marques-Silva, J. (2019).
\newblock On relating explanations and adversarial examples.
\newblock {\em Advances in Neural Information Processing Systems}, 32.

\bibitem[Kanamori et~al., 2020]{Kanamori2020}
Kanamori, K., Takagi, T., Kobayashi, K., and Arimura, H. (2020).
\newblock {DACE}: Distribution-aware counterfactual explanation by
  mixed-integer linear optimization.
\newblock In {\em International Joint Conference on Artificial Intelligence},
  pages 2855--2862.

\bibitem[Karimi et~al., 2022]{Karimi2022}
Karimi, A.-H., Barthe, G., Sch{\"o}lkopf, B., and Valera, I. (2022).
\newblock A survey of algorithmic recourse: contrastive explanations and
  consequential recommendations.
\newblock {\em ACM Computing Surveys}, 55(5):1--29.

\bibitem[Lakshminarayanan et~al., 2017]{Lakshminarayanan2017}
Lakshminarayanan, B., Pritzel, A., and Blundell, C. (2017).
\newblock Simple and scalable predictive uncertainty estimation using deep
  ensembles.
\newblock {\em Advances in Neural Information Processing Systems}, 30.

\bibitem[Laugel et~al., 2019]{Laugel2020}
Laugel, T., Lesot, M.-J., Marsala, C., Renard, X., and Detyniecki, M. (2019).
\newblock The dangers of post-hoc interpretability: Unjustified counterfactual
  explanations.
\newblock In {\em International Joint Conference on Artificial Intelligence},
  pages 2801--2807.

\bibitem[Liu et~al., 2008]{Liu2008}
Liu, F.~T., Ting, K.~M., and Zhou, Z.-H. (2008).
\newblock Isolation forest.
\newblock In {\em Eighth IEEE International Conference on Data Mining}, pages
  413--422. IEEE.

\bibitem[Lucic et~al., 2022]{Lucic2022}
Lucic, A., Oosterhuis, H., Haned, H., and de~Rijke, M. (2022).
\newblock Focus: Flexible optimizable counterfactual explanations for tree
  ensembles.
\newblock In {\em Proceedings of the AAAI Conference on Artificial
  Intelligence}, volume 36, no. 5, pages 5313--5322.

\bibitem[Luedtke and Ahmed, 2008]{Luedtke2008}
Luedtke, J. and Ahmed, S. (2008).
\newblock A sample approximation approach for optimization with probabilistic
  constraints.
\newblock {\em SIAM Journal on Optimization}, 19(2):674--699.

\bibitem[Pagnoncelli et~al., 2009]{Pagnoncelli2009}
Pagnoncelli, B.~K., Ahmed, S., and Shapiro, A. (2009).
\newblock Sample average approximation method for chance constrained
  programming: theory and applications.
\newblock {\em Journal of Optimization Theory and Applications},
  142(2):399--416.

\bibitem[Parmentier and Vidal, 2021]{Parmentier2021}
Parmentier, A. and Vidal, T. (2021).
\newblock Optimal counterfactual explanations in tree ensembles.
\newblock In {\em International Conference on Machine Learning}, pages
  8422--8431. PMLR.

\bibitem[Pawelczyk et~al., 2022]{Pawelczyk2022}
Pawelczyk, M., Agarwal, C., Joshi, S., Upadhyay, S., and Lakkaraju, H. (2022).
\newblock Exploring counterfactual explanations through the lens of adversarial
  examples: A theoretical and empirical analysis.
\newblock In {\em International Conference on Artificial Intelligence and
  Statistics}, pages 4574--4594. PMLR.

\bibitem[Rawal et~al., 2020]{Rawal2020}
Rawal, K., Kamar, E., and Lakkaraju, H. (2020).
\newblock Can i still trust you?: Understanding the impact of distribution
  shifts on algorithmic recourses.
\newblock {\em arXiv preprint arXiv:2012.11788}.

\bibitem[Sagi and Rokach, 2018]{Sagi2018}
Sagi, O. and Rokach, L. (2018).
\newblock Ensemble learning: A survey.
\newblock {\em Wiley Interdisciplinary Reviews: Data Mining and Knowledge
  Discovery}, 8(4):e1249.

\bibitem[Shapiro et~al., 2014]{Shapiro2014}
Shapiro, A., Dentcheva, D., and Ruszczynski, A. (2014).
\newblock {\em Lectures on Stochastic Programming: Modeling and Theory}.
\newblock SIAM.

\bibitem[Tolomei et~al., 2017]{Tolomei2017}
Tolomei, G., Silvestri, F., Haines, A., and Lalmas, M. (2017).
\newblock Interpretable predictions of tree-based ensembles via actionable
  feature tweaking.
\newblock In {\em Proceedings of the 23rd ACM SIGKDD International Conference
  on Knowledge Discovery and Data Mining}, pages 465--474.

\bibitem[Upadhyay et~al., 2021]{Upadhyay2021}
Upadhyay, S., Joshi, S., and Lakkaraju, H. (2021).
\newblock Towards robust and reliable algorithmic recourse.
\newblock In {\em Advances in Neural Information Processing Systems}.

\bibitem[Ustun et~al., 2019]{Ustun2019}
Ustun, B., Spangher, A., and Liu, Y. (2019).
\newblock Actionable recourse in linear classification.
\newblock In {\em Conference on Fairness, Accountability, and Transparency},
  pages 10--19.

\bibitem[Verma et~al., 2020]{Verma2020}
Verma, S., Dickerson, J., and Hines, K. (2020).
\newblock Counterfactual explanations for machine learning: A review.
\newblock {\em arXiv preprint arXiv:2010.10596}.

\bibitem[Wachter et~al., 2017]{Wachter2017}
Wachter, S., Mittelstadt, B., and Russell, C. (2017).
\newblock Counterfactual explanations without opening the black box: Automated
  decisions and the {GDPR}.
\newblock {\em Harvard Journal of Law \& Technology}, 31:841.

\end{thebibliography}

\clearpage
\appendix
\section{Supplementary Material: Proofs}
\subsection{Proof of Lemma~\ref{lem:gN}}
\label{app:gNproof}
The c.d.f. of the binomial distribution can be expressed in terms of the regularized incomplete beta function $I_x(a, b)$ as $g_N(p) = I_{1-p}(N/2, N/2+1)$ (see \cite{Abramowitz1988}, Section 26.5). Since $I_x(a, b)$ is the c.d.f. of the beta distribution, $I_x(a, b)$ is increasing in $x$ and is invertible, so that $g_N(p)$ is decreasing in $p$ and is also invertible. 

\subsection{Proof of Proposition~\ref{th:fromForestToTree}}
\label{app:fromForestToTree}
We can deduce from Lemma~\ref{lem:gN} that $1 - B\left(N/2 ; N, p(x) \right) \ge 1-\alpha$ if and only if $p(x) \ge \pStar$. Further, the following relationship holds for any ensemble of $N$ base learners trained independently on $z_n$:
\begin{align*}
    \Prob_{\bm{\xi}}  ( h_N(x, \bm{\xi}) \ge 1/2 ) & = \Prob_{\bm{\xi}} ( \sum_{i=1}^N t(x, \xi) \ge N/2 )\\
    & = \Prob_{\bm{\xi}} ( \sum_{i=1}^N t(x, \xi) = N/2 ) + \Prob_{\bm{\xi}} ( \sum_{i=1}^N t(x, \xi) > N/2 )\\
    & = \Prob_{\bm{\xi}} ( \sum_{i=1}^N t(x, \xi) = N/2 ) + 1 - B(N/2 ; N, p(x) )\\
    & \ge 1 - B(N/2 ; N, p(x) ).
\end{align*}
When $N$ is odd, $\Prob_{\bm{\xi}} \left( \sum_{i=1}^N t(x, \xi) = N/2 \right) = 0$ and the two conditions are equivalent.

\subsection{Proof of Lemma~\ref{lem:p_var_with_N}}
\label{app:proof_of_p_var_with_N}
Lemma~\ref{lem:p_var_with_N} is obtained by studying the sequence $f_N: N \mapsto B\left(N/2 ; N, p \right)$ with fixed~$p$. Interestingly, this sequence is not monotonic in $N$ and we need to analyze separately the cases of odd and even integers. We will need the following lemma.
\begin{lemma}
    \label{lem:p_N_smaller}
    Given $N \in \mathbb{N}, \, \alpha \le 1/2 \Rightarrow p_{N,\alpha}^* \ge 1/2 .$
\end{lemma}
\begin{proof}
    Given $\alpha \in [0, 1/2]$, we know from Lemma~\ref{lem:gN} that $p_{N,\alpha}^* \ge p_{N,1/2}^*$. We have
    \begin{align*}
        p_{N,1/2}^* \ge 1/2 \iff g_N(p_{N,1/2}^*) \le g_N(1/2) \\
         \iff 1/2 \le B(N/2;N,1/2),
    \end{align*}
    since $g_N$ is decreasing. Notice that the latest condition is true for all $N$ since $B(N/2;N,1/2) = 1/2$ when $N$ is odd and $B(N/2;N,1/2) < 1/2$ when $N$ is even.
\end{proof}
    
\textbf{Proof of Lemma~\ref{lem:p_var_with_N}(a).} We focus first on odd integers. Define $S_N \sim \Binom(N, p)$ and let $m \in \mathbb{N}$. The c.d.f. of the binomial distribution is:
\begin{align*}
    B( m+1 ; 2m+3,& p) =  \Prob (S_{2m+3} \le m+1) \\
    = & \Prob (S_{2m+3} \le m+1 \mid S_{2m+1} \le m-1) \Prob (S_{2m+1} \le m-1)\\
    & \quad + \Prob (S_{2m+3} \le m+1 \mid S_{2m+1} = m) \Prob (S_{2m+1} = m)\\
    & \quad + \Prob (S_{2m+3} \le m+1 \mid S_{2m+1} = m + 1) \Prob (S_{2m+1} = m + 1) \\
    = & \Prob (S_{2m+1} \le m - 1) + \left( 1 - p^2 \right) \Prob (S_{2m+1} = m) \\
    & \quad + (1-p)^2 \Prob (S_{2m+1} = m + 1).
\end{align*}
Notice that
\begin{equation}
    \label{eq:partial_res1}
    \Prob (S_{2m+1} \le m - 1) = \Prob (S_{2m+1} \le m) - \Prob (S_{2m+1} = m)
\end{equation} and, because of the symmetry of the binomial coefficients,
\begin{equation}
    \label{eq:partial_res2}
    (1-p)^2 \Prob (S_{2m+1} = m + 1) = p(1-p) \Prob (S_{2m+1} = m).
\end{equation}
By combining Equations~\eqref{eq:partial_res1} and \eqref{eq:partial_res2}, the difference between $f_{2m+3}$ and $f_{2m+1}$ can be expressed as
\begin{equation*}
    B(m+1 ; 2m+3, p) - B(m ; 2m+1, p) = p(1-2p) \Prob (S_{2m+1} = m),
\end{equation*}
which is negative if and only if $p \ge 1/2$.

To finish the proof, let $\alpha \le 1/2$ and notice that, given the definition of $p_{N, \alpha}^*$, we have
\begin{equation}
    \label{eq:alpha_F_p_star}
    \alpha = B(m+1 ; 2m+3, p_{2m+3,\alpha}^*) = B(m ; 2m+1, p_{2m+1,\alpha}^*).
\end{equation}
Since $\alpha \le 1/2$, we know from Lemma~\ref{lem:p_N_smaller} that $p_{2m+1,\alpha}^* \ge 1/2$ and therefore:
\begin{equation}
    \label{eq:F_res_a}
    B(m ; 2m+1, p_{2m+1,\alpha}^*) \ge B(m+1 ; 2m+3, p_{2m+1,\alpha}^*).
\end{equation}
By combining Equations~\eqref{eq:alpha_F_p_star} and \eqref{eq:F_res_a}, we obtain
\begin{equation*}
    B(m+1 ; 2m+3, p_{2m+3,\alpha}^*) \ge B(m+1 ; 2m+3, p_{2m+1,\alpha}^*),
\end{equation*}
which is only true when $p_{2m+1,\alpha}^* \ge p_{2m+3,\alpha}^*$ since $g_N$ is decreasing according to Lemma~\ref{lem:gN}.

\textbf{Proof of Lemma~\ref{lem:p_var_with_N}(b).} We focus now on even integers. We can similarly determine the difference between $f_{2(m+1)}$ and $f_{2m}$ as
\begin{align*}
    B & (m+1 ; 2(m+1), p) - B(m ; 2m, p)\\
     & = (1-p)^2 \Prob (S_{2m} = m+1) - p^2 \Prob (S_{2m} = m),
\end{align*}
so that
\begin{align*}
    f_{2(m+1)} - f_{2m} = & \, p^{m+1}(1-p)^m \left[ (1-p)\binom{2m}{m+1} - p\binom{2m}{m} \right]\\
    = & \, p^{m+1}(1-p)^m (2m)! \left[ \frac{1-p}{(m+1)!(m-1)!} - \frac{p}{(m!)^2}  \right] \\
    = & \, p^{m+1}(1-p)^m (2m)! \frac{\text{Num}(p)}{(m+1)!(m-1)!(m!)^2}
\end{align*}
where $\text{Num}(p) = m\cdot(m-1)! \left[(1-p)m - p(m+1) \right]$. Thus, $f_{2m}$ is greater than $f_{2(m+1)}$ if and only if $\text{Num}(p) \le 0$ that is
\begin{equation*}
    p \ge \frac{m}{2m+1}.
\end{equation*}
Since $m/(2m+1) < 1/2$ for all $m > 0$, $B(m+1 ; 2(m+1), p) \le B(m ; 2m, p)$ when $p \ge 1/2$.

To finish the proof, let $\alpha \le 1/2$ be such that:
\begin{equation}
    \label{eq:alphaPLemma2bproof}
    \alpha = B(m+1 ; 2(m+1), p_{2(m+1), \alpha}^*) = B(m ; 2m, p_{2m, \alpha}^*).
\end{equation}
Since $B(m+1 ; 2(m+1), p_{2(m+1), \alpha}^*) \le B(m ; 2m, p_{2(m+1), \alpha}^*)$, we have from~\eqref{eq:alphaPLemma2bproof} that:
\begin{equation*}
    B(m ; 2m, p_{2m, \alpha}^*) \le B(m ; 2m, p_{2(m+1), \alpha}^*),
\end{equation*}
which is true when $p_{2m, \alpha}^* \ge p_{2(m+1), \alpha}^*$ since $g_N$ is decreasing according to Lemma~\ref{lem:gN}.

\textbf{Proof of Lemma~\ref{lem:p_var_with_N}(c).} Let $m \in \mathbb{N}$, we have
\begin{align*}
    B (m ; 2m+1, p) = &\Prob (S_{2m+1} \le m) \\
    = & \Prob (S_{2m+1} \le m \mid S_{2m} \le m-1) \Prob (S_{2m} \le m-1)\\
    & \quad + \Prob (S_{2m+1} \le m \mid S_{2m} = m) \Prob (S_{2m} = m)\\
    = & \Prob (S_{2m} \le m-1) + (1-p) \Prob (S_{2m} = m).
\end{align*}
Because
\begin{equation*}
    \Prob (S_{2m} \le m) = \Prob (S_{2m} \le m-1) + \Prob (S_{2m} = m),
\end{equation*}
the difference
\begin{equation*}
    f_{2m+1} - f_{2m} = - p \Prob (S_{2m} = m)
\end{equation*}
is negative for all $p \in [0, 1]$. Hence, let $\alpha$ be such that:
\begin{equation}
    \label{eq:alphaPLemma2cProof}
    \alpha = B(m ; 2m, p_{2m, \alpha}^*) = B(m ; 2m+1, p_{2m+1, \alpha}^*).
\end{equation}
Since $B(m ; 2m+1, p_{2m, \alpha}^*) \le B(m ; 2m, p_{2m, \alpha}^*)$, we have from~\eqref{eq:alphaPLemma2cProof} that:
\begin{equation*}
    B(m ; 2m+1, p_{2m, \alpha}^*) \le B(m ; 2m+1, p_{2m+1, \alpha}^*),
\end{equation*}
which is true when $p_{2m, \alpha}^* \ge p_{2m+1, \alpha}^*$ since $g_N$ is decreasing according to Lemma~\ref{lem:gN}.

\section{Supplementary Material: Implementation Details}
\label{app:implementation}
In this section, we detail the implementation of our methods as well as the formulation of the plausibility-based benchmarks.

\subsection{SAA-based Methods}
The SAA-based methods such as Direct-SAA, Robust-SAA and the naive benchmark are implemented directly using the formulation of \cite{Parmentier2021} and its openly available implementation. The constraint on the target classification score in \eqref{opt:ce:target} is adapted for each method: 
\begin{itemize}[topsep=-0.5\parskip,
                  noitemsep]
    \item $\tau = 1/2$ for the naive benchmark,
    \item $\tau = \pStar$ for the Direct-SAA, and
    \item $\tau =\rho_{N, \alpha, \beta}^*$ for the Robust-SAA.
\end{itemize}

\textbf{Infeasibility.} When the robustness target is high, the resulting optimization problem might be infeasible, for instance, due to the actionability constraints. Thus, we relax Constraint~\eqref{opt:ce:target} and add a penalty cost $z_{\textit{pen}} = 5000 \cdot d$ where $d$ is the number of features. The objective of Problem~\eqref{opt:ce} is then:
\begin{equation*}
    f(x, x_{n+1}) + z_{\textit{pen}} \cdot \nu,
\end{equation*}
where $\nu$ is the relaxation term, and the relaxed constraint is:
\begin{equation*}
     h_N^0(x, \bm{\xi}^0) \ge \tau - \nu.
\end{equation*}

\subsection{RobX}
The RobX algorithm has been developed by \cite{Dutta2022} for providing explanations that are robust to an evolving data set. Their experiments target XGBoost ensembles, which are not randomized. Hence, the robustness of explanations obtained with RobX to algorithmic uncertainty has not been studied so far.

Since there is no available public implementation of RobX, we re-implemented the algorithm from scratch. Our implementation covers the edge cases that were unspecified in \cite{Dutta2022}, e.g., if the target class is not $0$ but $1$, or whenever no stable neighbors or no stable explanations are found after a maximum number of steps. We noted that using the algorithm as described in \cite{Dutta2022} does not lead to obtaining evenly spaced explanations from the initial naive explanation to its stable neighbors. We guessed that it was likely the result of a typo in the paper pseudo-code and therefore slightly revised the iterative algorithm to interpolate between the initial naive explanation and its stable neighbors

We have tried multiple configurations of the hyperparameters in a grid search. For the final experiments generating the results reported in our response, we use the number of perturbation samples and the standard deviation of the perturbation as recommended by \cite{Dutta2022} (i.e., $K=1000$ and $\sigma=0.1$). We set the number of neighbors to $c=10$ and the step size to $\alpha=0.05$ (no values are suggested in the original paper). The threshold parameter ($\tau$ in the notation of \citet{Dutta2022}) varies with values in $\{0.1, 0.2, 0.3, 0.4, 0.5, 0.6, 0.7\}$.

\subsection{Isolation Forests} Isolation forests have been used by \cite{Parmentier2021} to ensure that the counterfactual explanation is not an outlier of the target sample distribution. Isolation forests measure the anomaly of an observation by comparing the path length of the isolation forest to the average path length of a binary search tree. This plausibility-based benchmark is implemented as follows. First, an isolation forest made of $N_{\textsc{if}}=50$ trees is trained on the samples of the training set with the target class. The isolation forest is trained using scikit-learn. The contamination parameter is varied as explained in Section~\ref{sec:numStudy} and the other parameters are kept at their default values. The contamination parameter impacts the conservativeness of the model: the higher the contamination parameter, the higher the percentage of samples of the distribution that are classified as outliers. Then, the counterfactual explanation model is solved with additional constraints to state that the counterfactual is not classified as an outlier by the isolation forest. 

\allowdisplaybreaks
We modify the formulation of \cite{Parmentier2021} to allow varying contamination parameters. Let $\mathcal{T}_{\textsc{if}}$ be the sets of isolation trees and $\mathcal{V}_t^L$ the set of leaf nodes of tree $t$. Denote by $\delta_{v}$ the depth of leaf node $v$ and by $c(n)$ the average path length of a binary search tree using $n$ samples. Following the implementation of \cite{Parmentier2021}, the decision variable indicates $y_{v,t}$ is equal to $1$ if the counterfactual explanation ends in leaf node $v$ of tree $t$. The plausibility constraints are implemented as:
\begin{align}
    & I_t  = \sum_{v \in \mathcal{V}_t^L} (\delta_{v} + c(n_{v,t})) y_{v,t}, \quad \forall t \in \mathcal{T}_{\textsc{if}} \label{eq:ifDepth} \\
    & A = \frac{1}{N_{\textsc{if}}} \sum_{t \in \mathcal{T}_{\textsc{if}}} \frac{I_t}{c(n)} \label{eq:ifAnom},\\ 
    & A \le \log(-I_{\textsc{off}}) / \log(2) + \nu \label{eq:ifCons},\\
    & I_t, A, \nu \ge 0.
\end{align}
Constraint~\eqref{eq:ifDepth} measures the depth of the isolation trees, where $n_{v,t}$ is the number of samples in leaf node $v$ of tree $t$. Constraint~\eqref{eq:ifAnom} measures the anomaly of the counterfactual and Constraint~\eqref{eq:ifCons} restricts the anomaly of the counterfactual to be below the classification threshold of the trained isolation forest. The parameter $I_{\textsc{off}}$ is the forest offset that depends on the contamination parameter, following the scikit-learn implementation. As for the SAA-based methods, $\nu$ is a decision variable that relaxes the problem in case of infeasibility and is penalized with a large cost in the objective.

\subsection{Local Outlier Factor}
The local outlier factor~(LOF) is an outlier detection method used by \cite{Kanamori2020} to ensure the plausibility of counterfactual explanations. It is based on measuring the proximity of an observation to a distribution by comparing the distance of samples to their nearest neighbors. \cite{Kanamori2020} use the 1-LOF, a simpler implementation of the general k-LOF method in which only the nearest neighbor is considered.

The 1-LOF is implemented as a penalty term added to the objective function of Problem~\eqref{opt:ce} as:
\begin{equation*}
    f(x, x_{n+1}) + \sum_{i=1}^{n} l^{(i)} \cdot \rho_i,
\end{equation*}
where $l^{(i)}$ is the local reachability density of the training sample $x_i$ and $\rho_i$ is a decision variable that measures the reachability distance of $x$. The local reachability density measures how close a sample is to its nearest neighbors. In the 1-LOF case, the local reachability density is expressed as $l^{(i)} = lrd(x_i) = \Delta(1NN(x_i))^{-1}$,
where $1NN(x_n)$ is the nearest neighbor of $x_n$ and $\Delta(x)$ is the distance between x and its nearest neighbor. The reachability distance of $x$ and $x_i$ is the maximum between the distance between $x$ and $x_i$ and the distance between $x_i$ and its nearest neighbor.

Finding the reachability distance of $x$ is implemented through the following constraints:
\begin{align*}
    & \sum_{i=1}^n \nu_i = 1, & \\*
    & f(x_{i_1}, x_{i_2}) \le (1 - \nu_{i_1}) \cdot d, & \forall i_1, i_2, \in \{1,\dots, n\}, \\
    & \rho_i \ge \Delta(x_i) \cdot \nu_i, & \forall i \in \{1,\dots, n\}, \\
    & \rho_i \ge f(x, x_{i}) - (1 - \nu_i)  \cdot d, & \forall i \in \{1,\dots, n\}, \\
    & \nu_i \in \{0, 1\}, \rho_i \ge 0 & \forall i \in \{1,\dots, n\}.
\end{align*}
The binary variable $\nu_i$ tracks the nearest neighbor of $x$ in the training set. This method does not scale to large problem instances. In particular, the number of constraints grows quadratically in the number of samples in the training set. Thus, we simplify the formulation by restricting the number of training samples considered. We implement the above constraints only for the $N_r = 10$ nearest neighbors of the initial observation $x_{n+1}$. This allows us to repeatedly obtain counterfactual explanations using the 1-LOF plausibility constraints on all data sets and for all penalty factors $\lambda$. Interestingly, the ''Online News Popularity'' data set is still challenging to solve to optimality for $\lambda \ge  0.1$. In this case, we further reduce the relative MIP optimality gap parameter of the solver to $10\%$. Thus, we can apply the 1-LOF method to larger and more diverse data sets than \cite{Kanamori2020}, who impose a time limit of $600$ seconds to solve the optimization model.

\section{Supplementary Material: Additional Experimental Results}
\label{app:addExp}
\subsection{Summary of Data Sets}
\label{app:datasets}
A summary of the data sets is given in Table~\ref{tab:datasets}, where $n$ is the size of the training set, $d$ the number of features, and $(d_B, d_C, d_D, d_N)$ the number of binary, categorical, discrete numerical, and continuous numerical features, respectively.
\begin{table}[ht!]
    \center
    \vspace{-5mm}
    \begin{sc}
	\caption{Summary of the data sets.}
	\label{tab:datasets}
        \begin{tabular}{lrrrrrrl}
    \toprule
    \textbf{Data set}      & $n$   & $d$ & $d_B$ & $d_C$ & $d_D$ & $d_N$ & Source     \\
    \midrule
    Adult (A)                 & 45222 & 11  & 2     & 4     &  3    & 2     & UCI        \\
    Compas (C)                & 5278  & 5   & 2     & 0     &  3    & 0     & ProPublica \\
    Credit Card Default (CC)  & 29623 & 14  & 3     & 0     &  11   & 0     & UCI        \\
    German Credit (GC)        & 1000  & 9   & 0     & 3     &  5    & 1     & UCI        \\
    Online News (ON)          & 39644 & 47  & 2     & 2     &  6    & 37    & UCI        \\
    Phishing (P)        & 11055 & 30  & 8     & 0     &  22   & 0     & UCI        \\
    Spambase (S)              & 4601  & 57  & 0     & 0     &   0   & 57    & UCI        \\
    Students Performance (SP) & 395   & 30  & 13    & 4     &  13   & 0     & UCI\\
    \bottomrule
\end{tabular}
    \end{sc}
    \vspace{-5mm}
\end{table}

\subsection{Computation Time}
We show the computation time to solve the optimization models of all the implemented methods in Figure~\ref{fig:runtime}. For each data set and for each method, we show the boxplot of the runtimes for all parameters (e.g. for all values of $\alpha$ in the Direct-SAA case). The figure shows that the Direct-SAA method has only a small increase in computation time compared to the naive benchmark. Conversely, the plausibility-based benchmarks require more computational effort to find explanations, especially for large data sets.
\begin{figure}[ht!]
	\centering
	\includegraphics[width=0.9\linewidth]{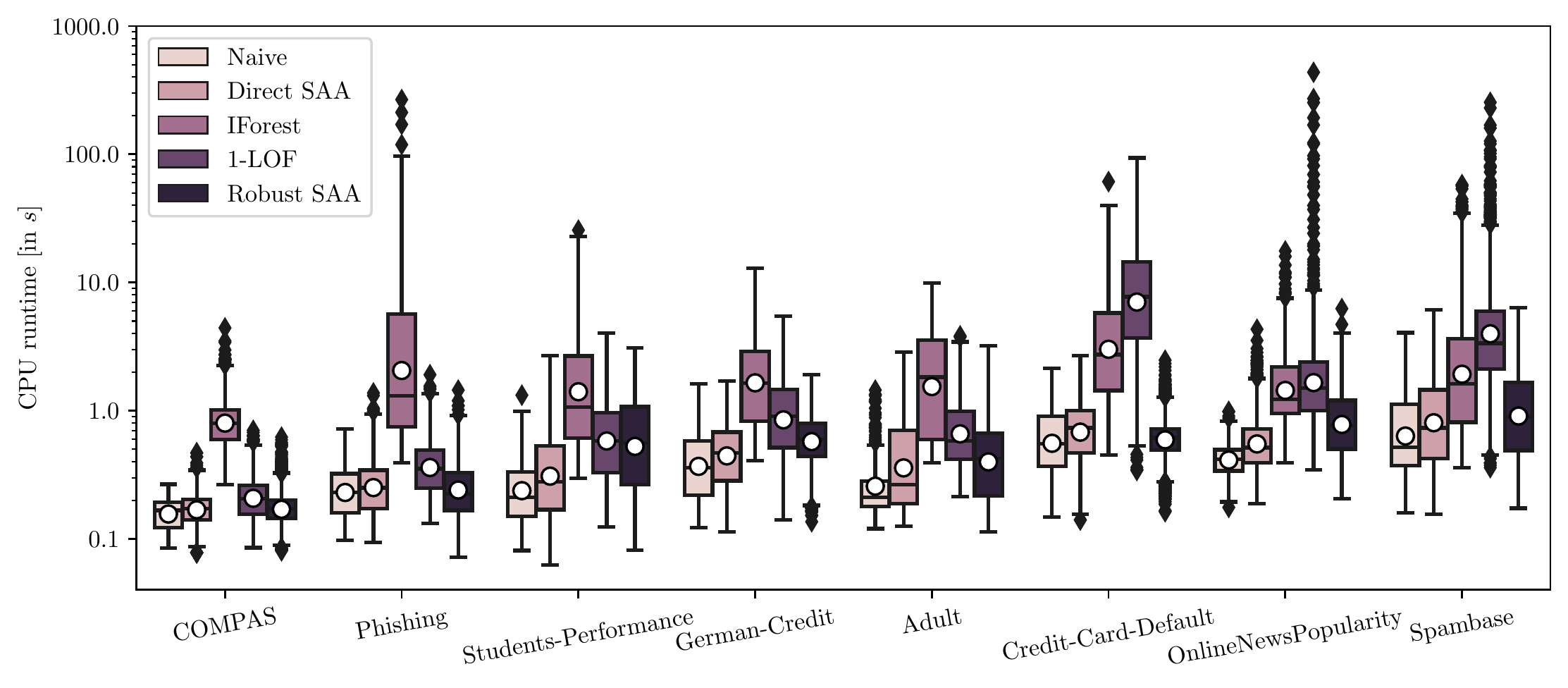}
	\caption{Boxplot of computation times over the different data sets.}
	\label{fig:runtime}
\end{figure}

\subsection{Examples of Counterfactual Explanations}
\label{app:cfTraj}
We provide additional examples of counterfactual explanations for increasing robustness levels. Figures~\ref{fig:adult-cfTraj} and \ref{fig:ccd-cfTraj} illustrate the counterfactual explanations obtained on the \textsc{Adult} and \textsc{Credit Card Default} data sets, respectively. Interestingly, the counterfactual explanations on the first data set are very sparse and involve only a few features. In contrast, the counterfactual explanations obtained on the second data set involve more diverse feature perturbations. It is also interesting to observe that the counterfactual explanations are sometimes not modified as the robustness level increases. This threshold effect is due to the discrete values of the modified features.
\begin{figure}[htb!]
    \centering
    \includegraphics[width=0.85\linewidth]{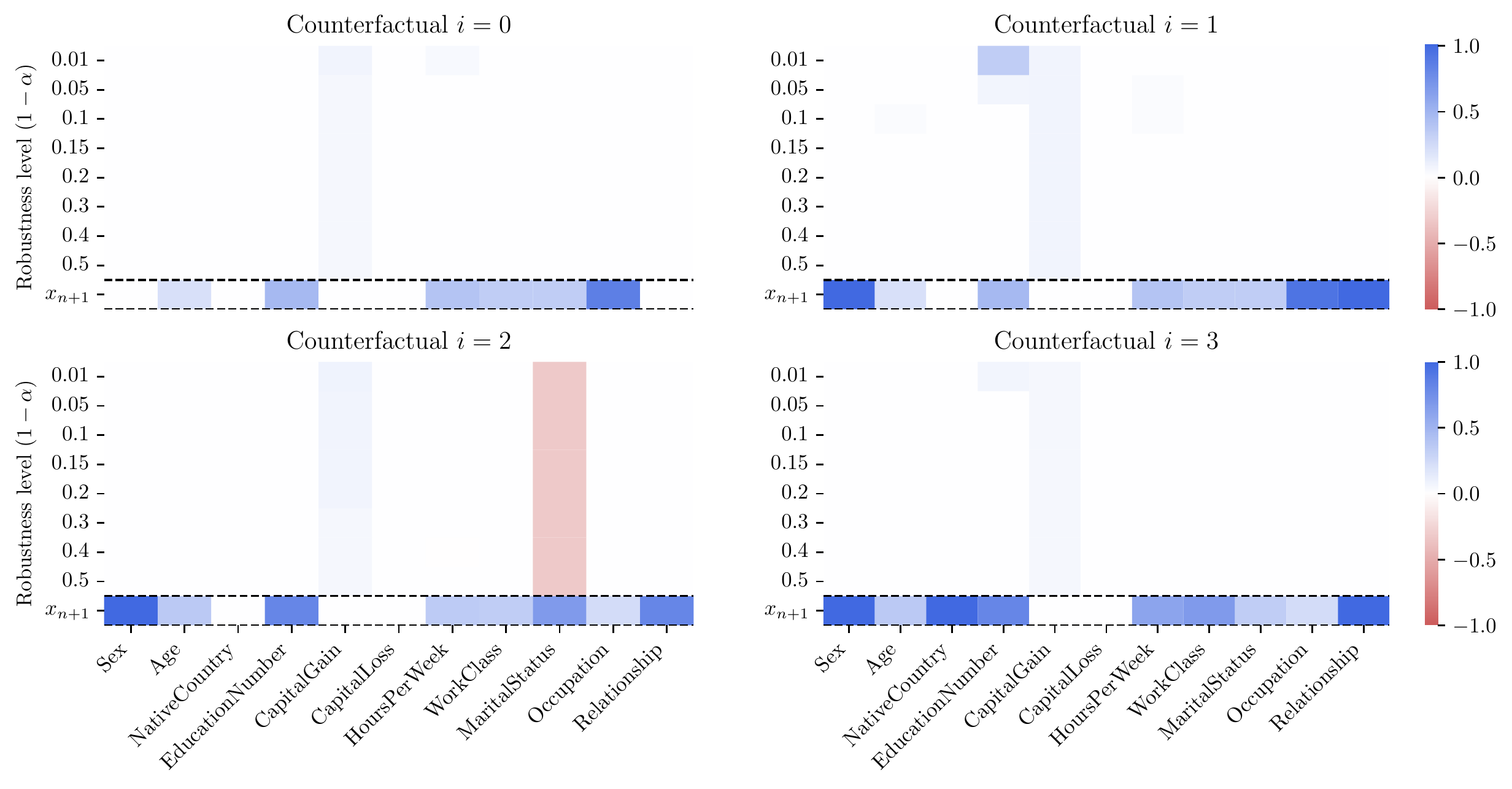}
    \caption{\textsc{Adult}: initial observation and counterfactual changes to reach the target class for increasing robustness level ($1-\alpha$).}
    \label{fig:adult-cfTraj}
\end{figure}

\begin{figure}[htb!]
    \centering
    \includegraphics[width=0.85\linewidth]{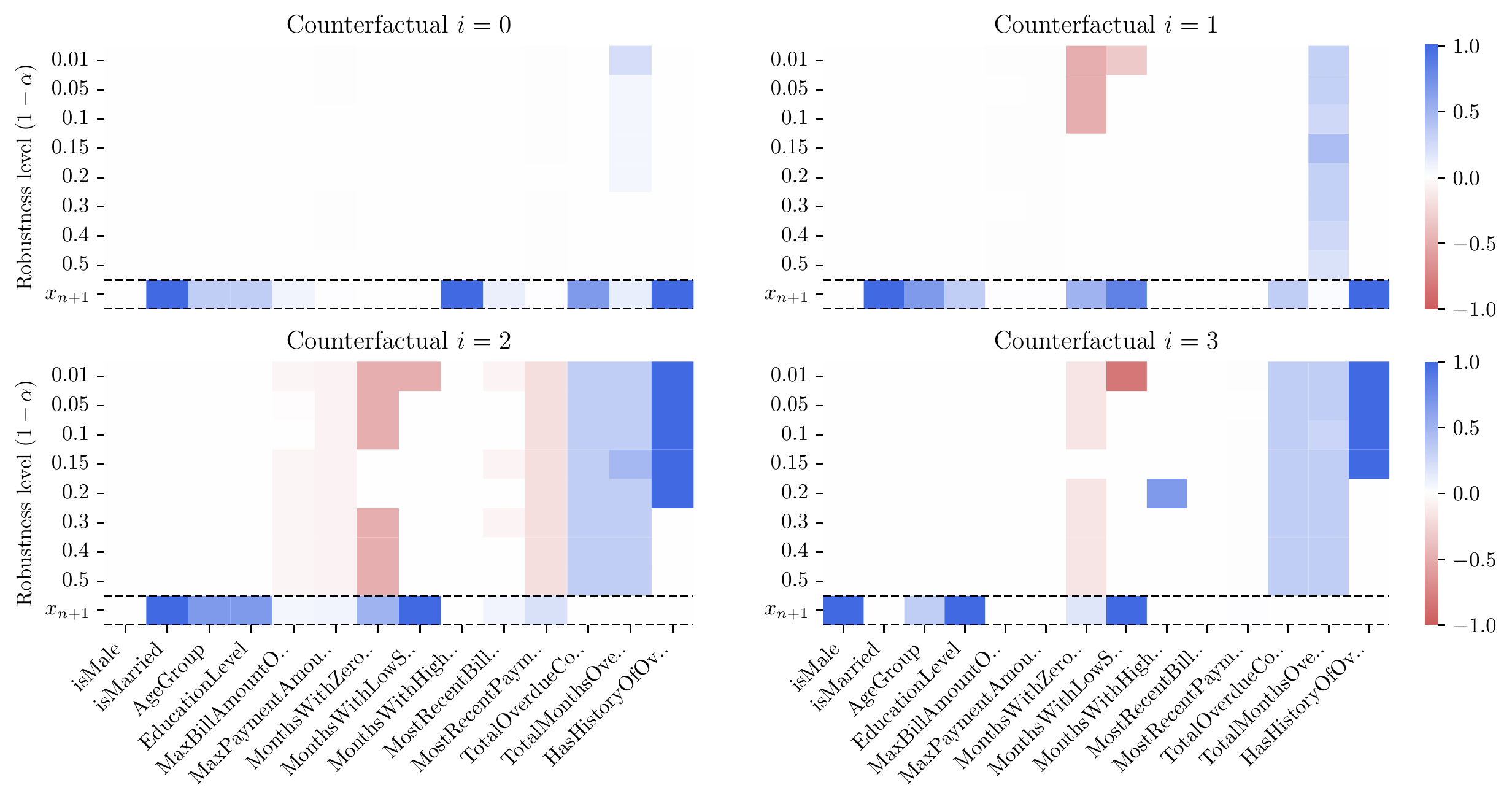}
    \caption{\textsc{Credit Card Default}: initial observation and counterfactual changes to reach the target class for increasing robustness level ($1-\alpha$).}
    \label{fig:ccd-cfTraj}
\end{figure}

\subsection{Robust Explanations of Large Forests}
\label{app:largeForests}

We perform the same experiment as in Section~\ref{sec:robExp}, but now using larger forests of size $N=400$. Due to the added computational effort, we repeat the experiments $S_{\textsc{SIM}}=20$ times for each simulation setting. We start by measuring the robustness of the naive explanations. The results are presented in Table~\ref{tab:naiveN400}. Although the average robustness of naive explanations increases slightly compared to the results with forests of size $N=100$, the results are overall similar. The average robustness of naive explanations is close to $50\%$ on many data sets, and falls below $25\%$ for the complex \textsc{Spambase} data set.
\begin{table}[ht!]
    \vspace{-1mm}
    \center
    \caption{Average naive explanations that remain valid when retraining a random forest with $N=400$ trees with fixed training data.}
    \label{tab:naiveN400}
    \begin{tabular}{lcccccccc}
    \toprule
    \textbf{Data set} &   A  &  C  &   CC  &  GC  &  ON  &  P  &  S   &   SP \\
    \midrule
    Validity [in \%]  &  59 &  92 &  39 &  47 &  49 &  89 &  24 &  53 \\
    \bottomrule
\end{tabular}
    \vspace{-3mm}
\end{table}

We now measure the robustness of explanations obtained through our methods. The results are presented in Figure~\ref{fig:algRobResults-validity-large-forests}. The performance of our method is essentially identical as in the case of $N=100$: the Direct-SAA method provides robust explanations at the desired level on all data sets except \textsc{Spambase}. In this case, the Robust-SAA method gives robust explanations for medium and high robustness targets using $\beta = 0.1$ and $\beta = 0.05$ respectively. This confirms that increasing the size of the random ensemble does not protect against algorithmic uncertainty when using the naive threshold $\tau = 1/2$ as in Problem~\eqref{opt:ce}.
\begin{figure}[ht!]
	\centering
	\resizebox{\linewidth}{!}{\input{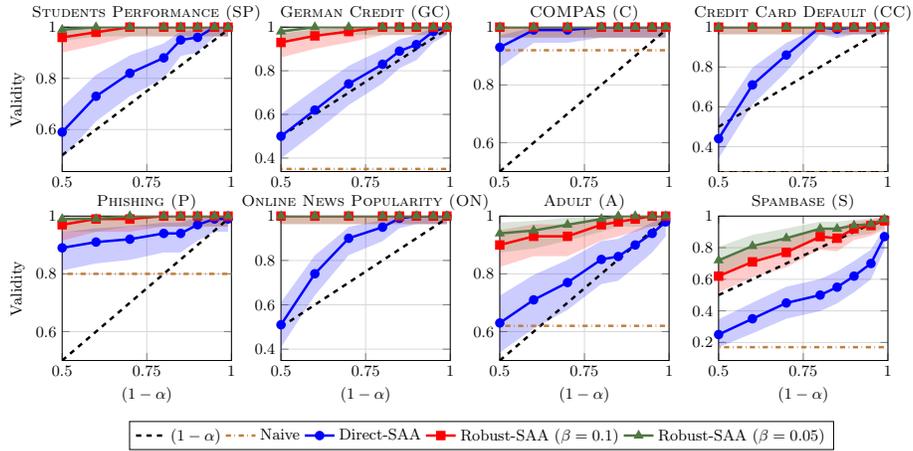}}
	\caption{Validity of robust counterfactuals for random forest with $N=400$ as a function of the robustness target $(1-\alpha)$.}
	\label{fig:algRobResults-validity-large-forests}
\end{figure}

\subsection{Experiments with an Evolving Data Set}
\label{app:evolvingData}
We assume throughout the paper that the training sample is fixed, we now relax this assumption and perform experiments with an evolving data set. Specifically, we consider the setting in which additional data is obtained before the model is retrained. For each data set, we sample $80\%$ of the data to train a first random forest and determine a counterfactual explanation. We then add the remaining $20\%$ of the data set and train a second random forest to validate the counterfactual explanation. The achieved validity of our methods is shown in Figure~\ref{fig:algRobResults-validity-evolving-data} as a function of the robustness target.

\begin{figure}[ht!]
	\centering
	\resizebox{\linewidth}{!}{\input{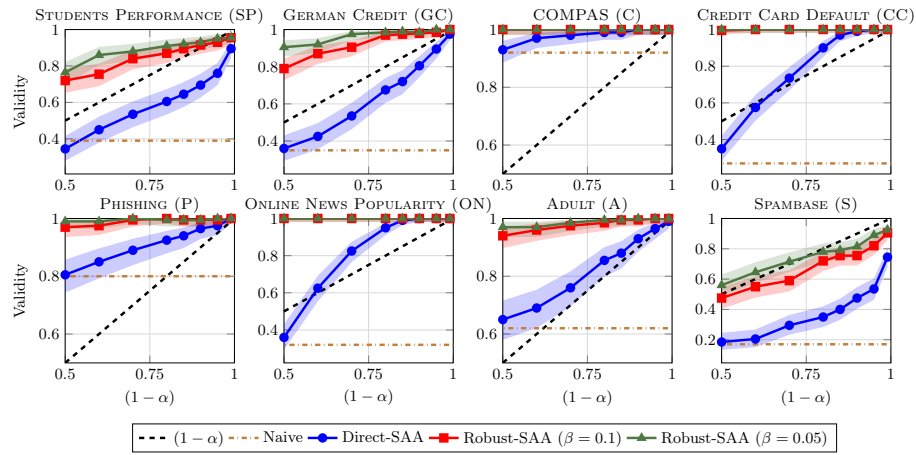}}
	\caption{Validity of robust counterfactuals as a function of the robustness target $(1-\alpha)$.}
	\label{fig:algRobResults-validity-evolving-data}
\end{figure}

Overall, the results are similar to the ones with fixed training data, as shown in Figure~\ref{fig:algRobResults-validity}. The Direct-SAA method satisfies the robustness target on $5$ out of $8$ data sets. The two new data sets for which the robustness target is not satisfied are \textsc{Student Performance} and \textsc{German Credit}. This can be well explained by the fact that these two data sets have the fewest data points ($395$ and $1000$ observations respectively). This suggests that the robustness of counterfactual explanations when the data sets evolve over time also depends on the size of the data set. The Robust-SAA method provides robust counterfactual explanations for these two data sets as well as for the \textsc{Spambase} data set. Hence, the Robust-SAA method is especially relevant when (1)~there are many features with high predictive importance, and (2)~the size of the data set is small and new data is observed before the model is retrained. The fact that the Direct-SAA performs well even with evolving data set is noteworthy: it shows that, when the size of the training set is large, algorithmic uncertainty dominates the sampling uncertainty of the evolving training data.

\subsection{Feature Importance, Sparsity, and Robustness}
\label{app:featImportanceAndSparsity}
We investigate the link between the feature changes from the counterfactual explanations and the predictive importance of the features. We measure the feature importance by using the permutation method of \cite{Breiman2001}. We also measure the average absolute changes to the features of counterfactual explanations given by the Direct-SAA method with robustness target $(1-\alpha) = 0.5$. We show the feature importance and average feature changes in Figures~\ref{fig:featureImportance1} and \ref{fig:featureImportance2} for two data sets with high robustness and two data sets with low robustness, respectively. The feature importance and changes are normalized to better compare the two measures. Remarkably, feature importance and the average feature changes are very similar on all data sets. Some of the main differences between feature importance and feature changes can be explained by the actionability constraints. For instance, the ''AgeGroup'' feature can only increase in the \textsc{Compas} data set. The data sets shown in Figure~\ref{fig:featureImportance1} have especially sparse important features. It is also for these two data sets that the robustness of naive explanations is the highest. This suggests that having a few features with high importance favors obtaining robust naive counterfactual explanations. Conversely, a data set with many important features is a good indicator that robustness should be explicitly considered when generating counterfactual explanations.
\begin{figure}[ht!]
    \centering
    \begin{subfigure}[b]{0.49\textwidth}
        \centering
        \includegraphics[width=\textwidth]{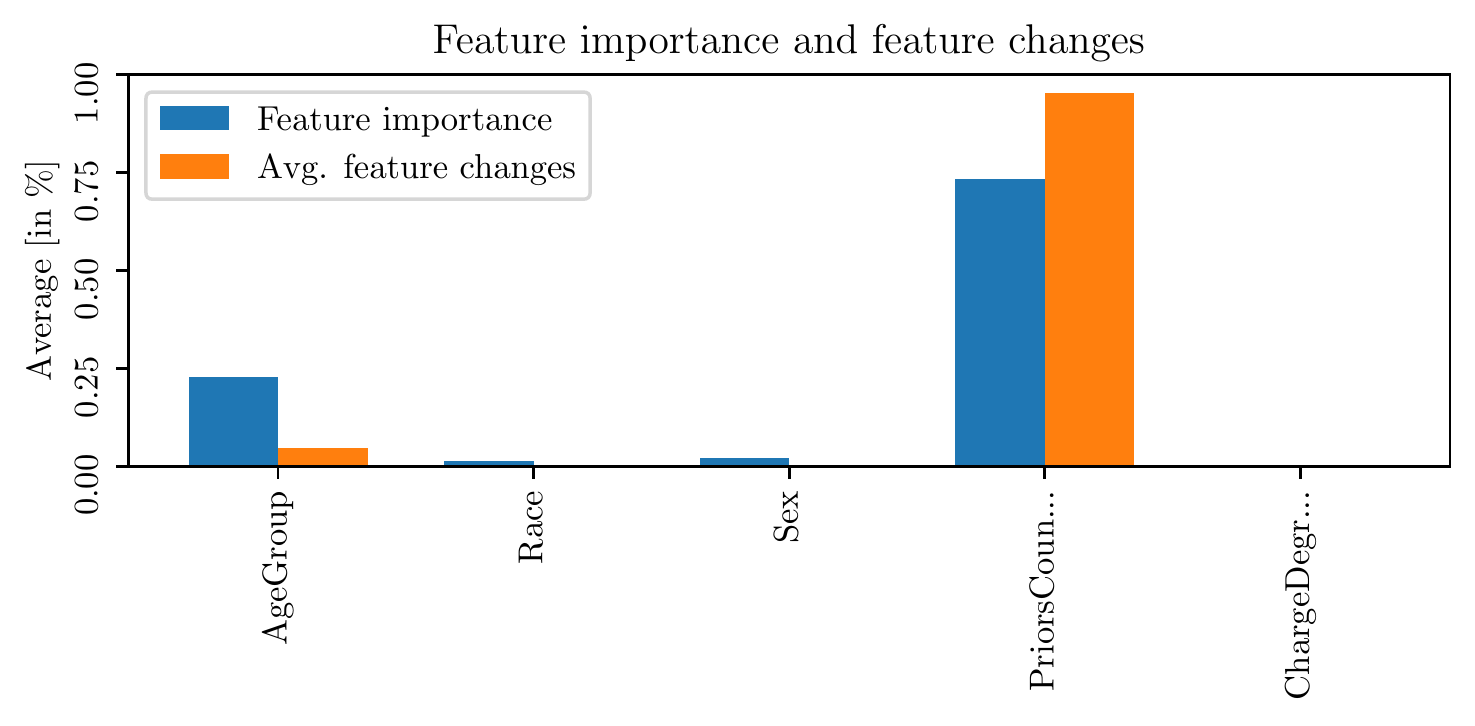}
        \caption[]
        {{\small \textsc{Compas}}}
    \end{subfigure}
    \begin{subfigure}[b]{0.49\textwidth}  
        \centering 
        \includegraphics[width=\textwidth]{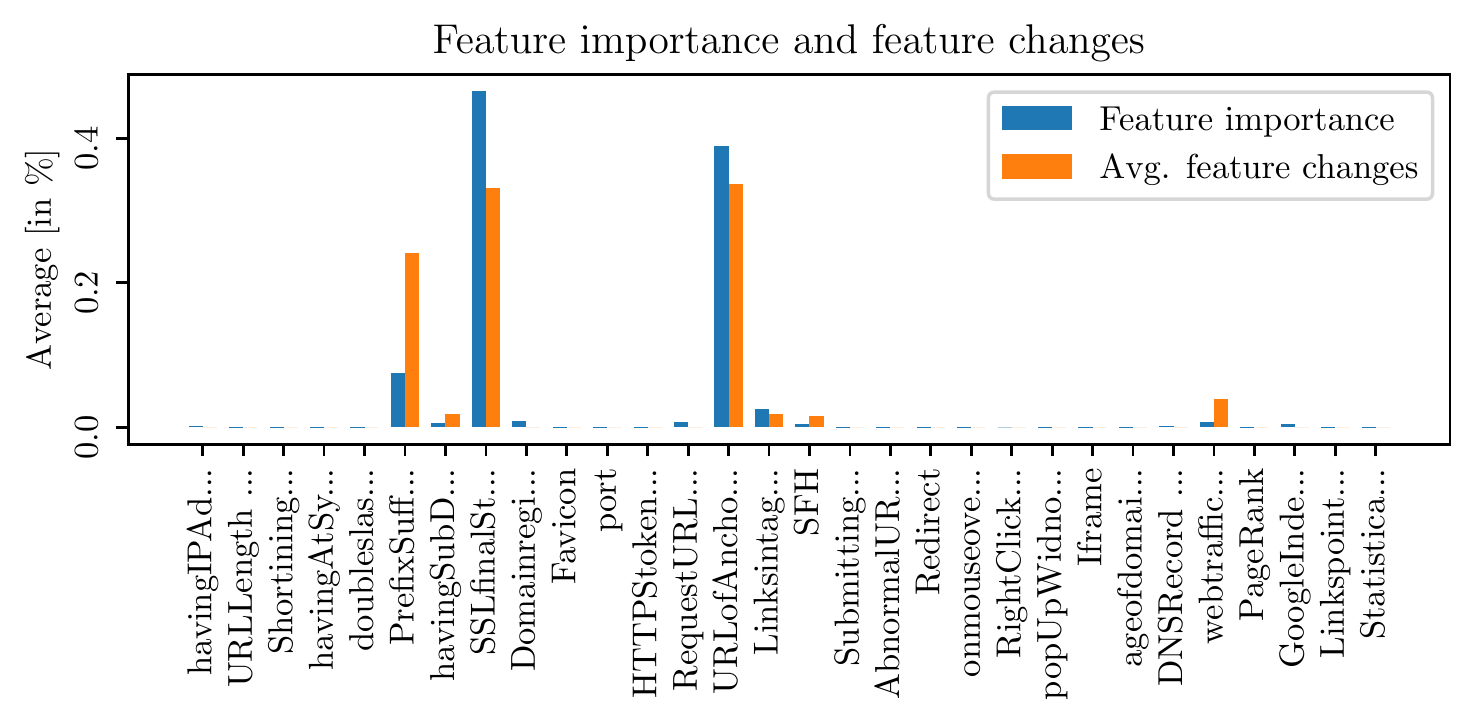}
        \caption[]%
        {{\small \textsc{Phishing}}}
    \end{subfigure}
    \caption[Feature importance and feature changes for two data sets with high robustness.]
    {\small Feature importance and feature changes for two data sets with high robustness.} 
    \label{fig:featureImportance1}
\end{figure}
\begin{figure}[ht!]
    \centering 
    \begin{subfigure}[b]{0.49\textwidth}   
        \centering 
        \includegraphics[width=\textwidth]{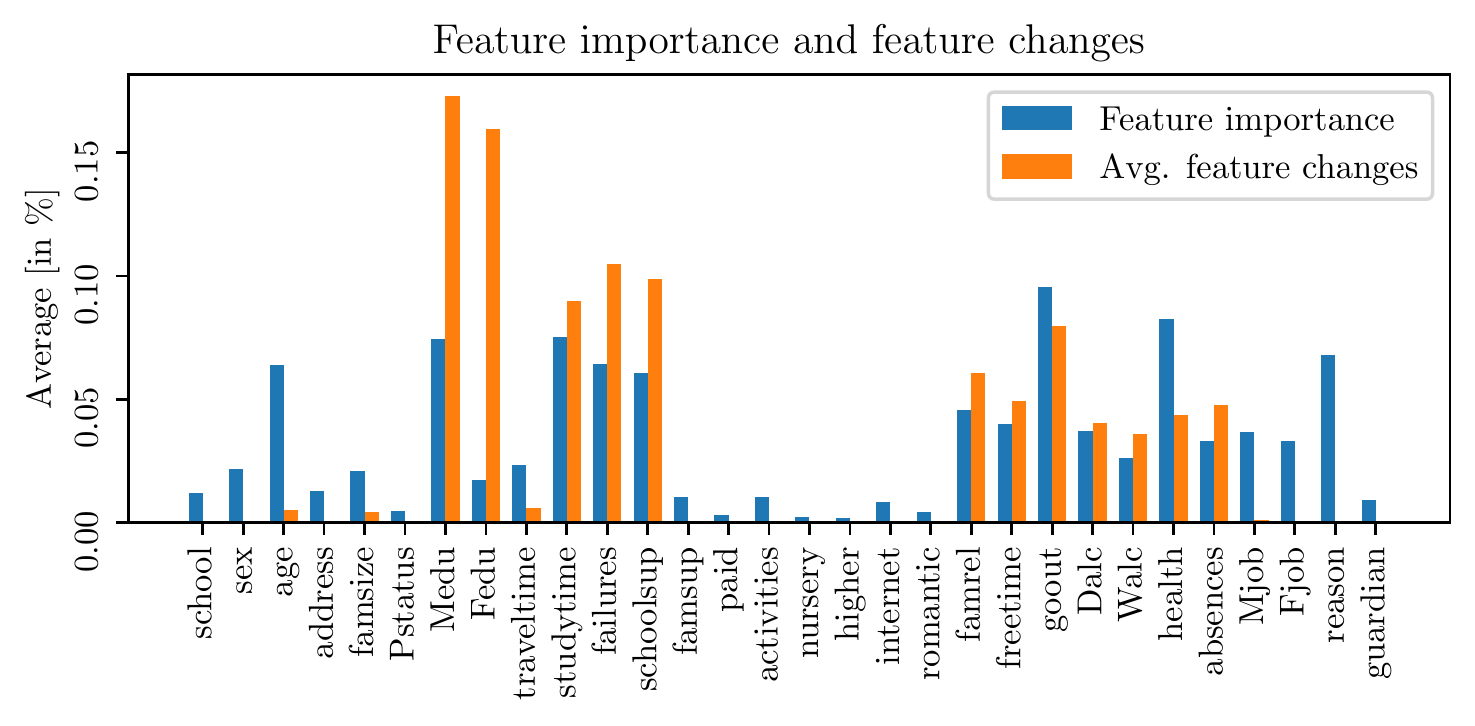}
        \caption[]%
        {{\small \textsc{Students Performance}}}
    \end{subfigure}
    \begin{subfigure}[b]{0.49\textwidth}   
        \centering 
        \includegraphics[width=\textwidth]{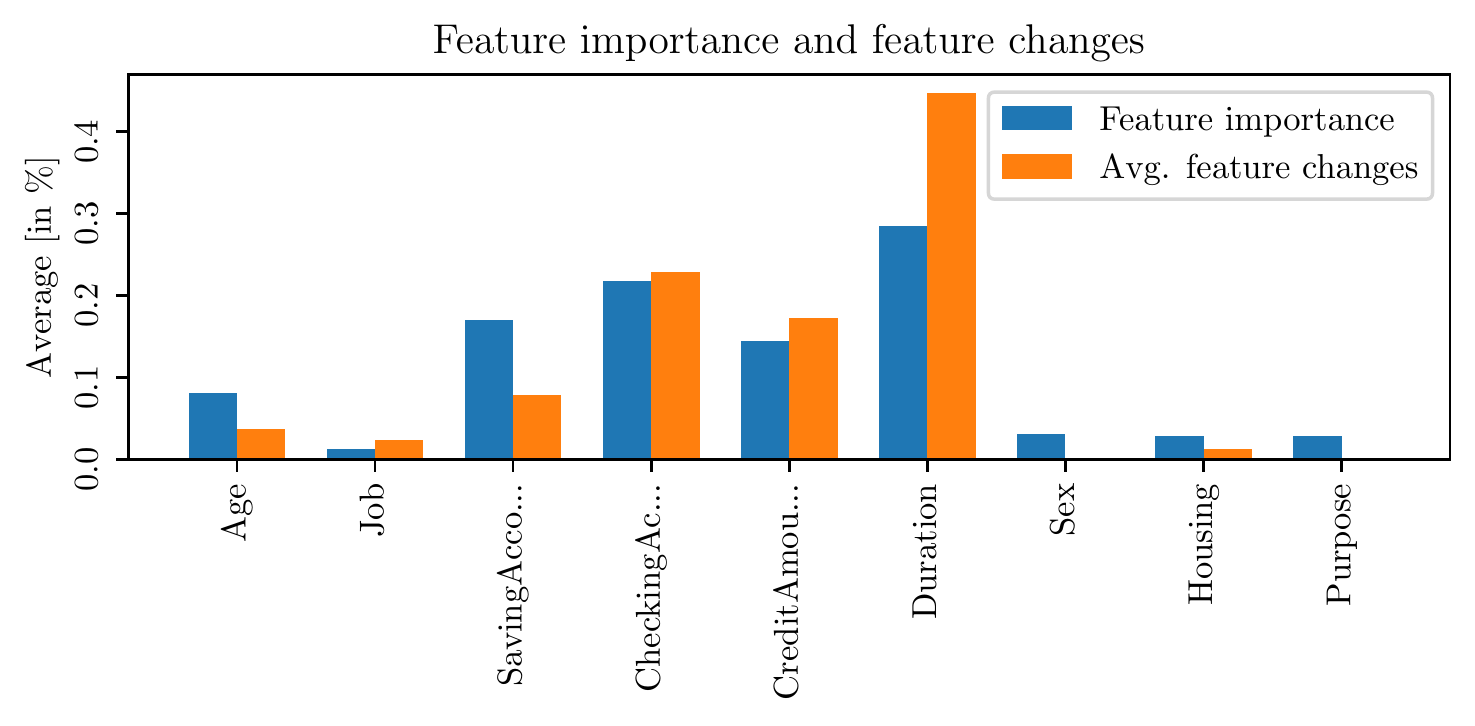}
        \caption[]%
        {{\small \textsc{German Credit}}}
    \end{subfigure}
    \caption[Feature importance and feature changes for two data sets with low robustness.]
    {\small Feature importance and feature changes for two data sets with low robustness.} 
    \label{fig:featureImportance2}
\end{figure}

\end{document}